%% file: main_arxiv.tex
\documentclass{article}
\usepackage[utf8]{inputenc}
\usepackage[colorlinks = true,
            linkcolor = blue,
            urlcolor  = blue,
            citecolor = blue,
            anchorcolor = blue]{hyperref}
\usepackage{amsmath,amssymb,amsfonts,amsthm,bm}
\usepackage{color,graphicx}
\usepackage[a4paper, total={6in,9in}]{geometry}

\usepackage{amssymb}
\usepackage{empheq}
\usepackage{mathtools}
\usepackage{graphicx}
\usepackage{bbm}
\usepackage{hyperref}
\usepackage{dirtytalk}
\usepackage{stmaryrd}
\usepackage{fancyhdr}
\usepackage{macros}

\usepackage{microtype}
\usepackage{graphicx}
\usepackage{subcaption}
\usepackage{booktabs}
\usepackage{makecell}
\usepackage{caption}
\usepackage{amsmath}
\usepackage{booktabs} 
\usepackage{tablefootnote}
\usepackage{footnote}
\usepackage{hyperref}
\usepackage{tabularx}
\usepackage[ruled,vlined]{algorithm2e}
\usepackage{epigraph}

\usepackage{verbatim}
\usepackage{listings}
\usepackage{algpseudocode}
\usepackage[affil-it]{authblk}
\usepackage{textcomp}
\usepackage{multirow}
\usepackage[toc,page]{appendix}
\usepackage[numbers, sort, comma, square]{natbib}
\usepackage{float}
\usepackage{todonotes}

\newif\ifshownotes
\shownotesfalse

\newif\ifarxiv
\arxivtrue

\newcommand{\arxivvspace}[1]{\ifarxiv\vspace{#1}\else\fi}
\newcommand{\neuripsvspace}[1]{\ifarxiv\else\vspace{#1}\fi}

\makeatletter
\let\oldsection\section
\renewcommand\section{\@ifstar{\starsection}{\nostarsection}}
\newcommand\nostarsection[1]{\vspace{-0.1in}\oldsection{#1}}
\newcommand\starsection[1]{\vspace{-0.1in}\oldsection*{#1}}

\makeatother

\usepackage{longtable}
\date{}
\newtheoremstyle{named}{}{}{\itshape}{}{\bfseries}{.}{.5em}{\thmnote{#1: #3}}
\theoremstyle{named}

\author{Konstantin Donhauser\thanks{equal contribution. Presented at the 35th Conference on Neural Information Processing Systems (NeurIPS 2021).} $^1$, Alexandru
Țifrea\printfnsymbol{1}$^1$, Michael Aerni$^1$\\
Reinhard Heckel$^{2,3}$ and Fanny Yang$^1$
}

\affil{$^1$ETH Zurich,
$^2$Rice University,
$^3$Technical University of Munich
}

\title{Interpolation can hurt robust generalization even when \\ there is no noise}

\begin{document}

\maketitle

\begin{abstract}

Numerous recent works show that overparameterization implicitly reduces variance
for min-norm interpolators and max-margin classifiers. These findings suggest
that ridge regularization has vanishing benefits in high dimensions.  We
challenge this narrative by showing that, even in the absence of noise, avoiding
interpolation through ridge regularization can significantly improve
generalization.  We prove this phenomenon for the robust risk of both linear
regression and classification, and hence provide the first theoretical result on
\emph{robust overfitting}.
\end{abstract}

\setcounter{page}{1}

\input{new_intro2}
\input{setting}

\input{lin_reg}
\input{log_reg}

\input{related_work}

\input{conclusion}

\newpage

\bibliographystyle{abbrvnat}

\bibliography{bibliography.bib}

\newpage
\appendix
\input{appendix/setting}

\input{appendix/experiments}

\input{appendix/linreg_exp}

\input{appendix/logreg_exp}

\input{appendix/linreg_theory}

\input{appendix/logreg_theory}

\end{document}

%% file: new_intro2.tex
\section{Introduction}

Conventional statistical wisdom
cautions the user who trains a model by minimizing a loss $\Loss(\theta)$:
if a global minimizer achieves zero or near-zero training loss (i.e.,
it \emph{interpolates}), we run the risk of overfitting (i.e., high
variance) and thus suboptimal prediction performance.  Instead,
\emph{regularization} is commonly used to reduce the effect of noise
and to obtain an estimator with better generalization.
Specifically, regularization limits model complexity and induces worse data
fit, for example via an explicit penalty term $R(\theta)$. The resulting
penalized loss $\Loss(\theta) + \lambda R(\theta)$ explicitly imposes certain structural
properties on the minimizer.
This classical rationale, however, does seemingly not apply to overparameterized models:
for example, large neural networks in practice exhibit good
generalization performance on \iid{} samples even if $\Loss(\theta)$
vanishes and label noise is present~\cite{Nakkiran20}.

Since interpolators are not unique in the overparameterized regime, it is
crucial to study the specific \emph{implicit biases} of interpolating estimators.
In particular, for common losses, a large body of recent work analyzes the properties of the
solutions found via gradient descent at convergence
(see
e.g.~\cite{Rosset04,Chizat18,Chizat20,Gunasekar18,Ji18,Ji19,Soudry18,Li2020}). For
example, for linear and logistic regression, it is well-known that
gradient descent converges to the min-$\ell_2$-norm and
max-$\ell_2$-margin solutions, respectively~\cite{Gunasekar18, Ji19,Lyu20,Soudry18}. These interpolating estimators also
minimize the respective penalized loss $\Loss(\theta) + \lambda
\|\theta\|^2_2$ in the limit of $\lambda\to 0$ \cite{Rosset04}.

A plethora of recent papers explicitly study generalization properties
on min-$\ell_2$-norm interpolators~\cite{Dobriban15,Ghorbani19,Hastie19,Bartlett20,Mei19, Muthukumar20,
    Muthukumar20a} and max-$\ell_2$-margin solutions
\cite{Deng20,Muthukumar20,Salehi19}, and show that the variance
decreases as the overparameterization ratio increases beyond the interpolation threshold.
While regularization with $\lambda >0$ is commonly known to reduce the risk at the
interpolation threshold~\cite{Hastie19, Nakkiran21}, many of these works are motivated by
the second descent of the double descent phenomenon~\cite{Belkin19b}
which suggests that regularization becomes redundant with sufficient overparameterization.
Hence, previous papers focus on highly overparameterized settings
where the optimal regularization parameter satisfies $\lambdaopt \leq
0$~\cite{Kobak20, Wu20, Richards21}, implying that for large $d/n$,
explicit regularization with $\lambda >0$ is redundant or even
detrimental for generalization.

\begin{figure*}[btp]
    \centering
    \begin{subfigure}[t]{\textwidth}
        \centering
        \begin{subfigure}{0.33\textwidth}
            \centering
            \includegraphics[width=1.8in]{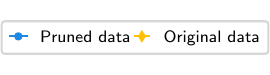}
        \end{subfigure}
        \begin{subfigure}{0.66\textwidth}
            \centering
            \includegraphics[width=2.4in]{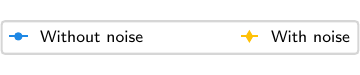}
        \end{subfigure}
    \end{subfigure}
    \begin{subfigure}[b]{\textwidth}
        \centering
        \begin{subfigure}{0.32\textwidth}
            \centering
            \includegraphics[width=1.8in]{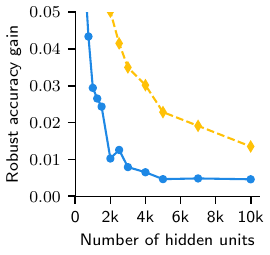}
            \caption{Neural networks}
            \label{fig:teaser_plot_nn}
        \end{subfigure}
        \begin{subfigure}{0.32\textwidth}
            \centering
            \includegraphics[width=1.8in]{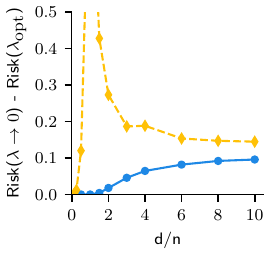}
            \caption{Linear regression}
            \label{fig:teaser_plot_linreg}
        \end{subfigure}
        \begin{subfigure}{0.32\textwidth}
            \centering
            \includegraphics[width=1.8in]{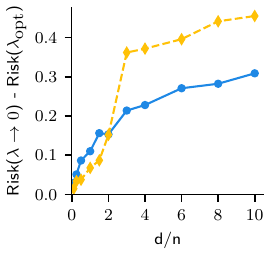}
            \caption{Linear classification}
            \label{fig:teaser_plot_logreg}
        \end{subfigure}
    \end{subfigure}
    \caption{
        Avoiding interpolation can benefit robustness even in the overparameterized ($d \gg n$) regime
        and for noiseless training data.
        We plot the robust accuracy gain of (a) early-stopped neural networks compared to models at convergence, fit on sanitized (binary 1-3) MNIST that arguably has minimal noise; and $\ell_2$ regularized estimators compared to interpolators with $\lambda \to 0$ for
        (b) linear regression with $n=10^3$ and
        (c) robust logistic regression with $n=10^3$.
        See \suppmat~\ref{sec:appendix_experiments} for experimental details
        and Sections~\ref{sec:linreg}~and~\ref{sec:logreg} for
        the precise settings of (b) and (c).
    }
    \neuripsvspace{-0.2in}
    \label{fig:teaser}
\end{figure*}

Taking a step back, this narrative
originated from theoretical and experimental findings that consider the \emph{standard}
test risk with identically distributed training and test data.
However, this measure cannot reflect the
\emph{robust risk} of models when the test data has a shifted
distribution, is attacked by adversaries, or contains many
samples from minority groups~\cite{Goodfellow15,Grother10,Quionero09,Zadrozny04}.
In fact, mounting empirical evidence suggests that regularization is indeed
helpful for robust generalization, even in highly overparameterized regimes
where the benefits for the standard risk are negligible~\cite{Rice20}.
This phenomenon is sometimes referred to as \emph{robust overfitting}.

In the presence of noise, the following intuition holds true:
since the robust risk amplifies estimation errors, its variance is larger
and hence regularization -- such as early stopping -- can be beneficial for generalization~\cite{Sanyal21}.
However, we observe that even when estimating entirely noiseless signals,
robust overfitting persists!
We observe this phenomenon in experiments with shallow neural networks on sanitized image data depicted in
Figure~\ref{fig:teaser_plot_nn} and, in fact, even for linear models trained on high-dimensional synthetic noiseless data.
In particular,
Figures~\ref{fig:teaser_plot_linreg},\ref{fig:teaser_plot_logreg} show that
min-$\ell_2$-norm and robust max-$\ltwo$-margin interpolators
(minimizers of the training loss for $\lambda \to 0$), 
achieve a higher robust risk than the corresponding regularized estimators that do not interpolate noiseless observations (minimizers for $\lambda>0$).



To date, our
observations in the noiseless case cannot be
explained by prior work.
On the contrary, they seem to contradict
a simple intuition: if the min-$\ell_2$-norm and robust max-$\ell_2$-margin
interpolators exhibit large risks as $\lambda\to 0$,
the induced bias for a small $\ell_2$-norm is potentially suboptimal
and a larger penalty weight $\lambda >0$
should only degrade performance.
In this paper, we provide possible explanations that debunk this intuition in the high-dimensional overparameterized regime.
We prove for isotropic Gaussian covariates
that a strictly positive regularization parameter $\lambda$
systematically improves robust generalization
for linear and robust logistic regression.
Empirically, we show that early stopping and other
factors  that lead to a non-interpolating estimator
achieve a similar effect.
Our results provide the first rigorous explanation
of robust overfitting even in the absence of noise.

In Section~\ref{sec:setting}, we formally define the setting that we use throughout our analysis.
We then first present precise asymptotic expressions for the robust risk for linear regression in Section~\ref{sec:linreg} that explicitly explain robust overfitting. Furthermore, in Section~\ref{sec:logreg}, we
consider classification with logistic regression and derive asymptotic results.



%% file: setting.tex
\section{Risk minimization framework}
\label{sec:setting}

In this section, we introduce
the data generating process that we assume
throughout our analysis, and define
the standard and robust risks that we use
as evaluation metrics.

\subsection{Problem setting}
\label{sec:data_model}

This paper considers the supervised learning problem
of estimating a mapping
from $d$-dimensional real-valued features $x\in \R^d$ to a target $y
\in \YY \subseteq \R$ given a training set of labeled samples 
$\Traindata = \{(\xind{i}, \y_i)\}_{i=1}^n$.
We assume that the feature vectors $x_i$ are drawn \iid{} from
the marginal distribution $\probx$ that we assume to be
an isotropic Gaussian.
We further focus on noiseless observations $y_i = \langle \thetatrue, x_i\rangle$ for regression tasks
and $y_i = \sgn(\langle \thetatrue, x_i \rangle)$ for classification tasks,
respectively. However, the main results are more general and apply to noisy observations as well.
For regression, we assume additive Gaussian noise with zero mean and
$\sigma^2$ variance, while for classification we
flip a certain percentage of the training labels.

%
%



This paper studies the high-dimensional asymptotic regime where $d/n \to \gamma$
as both the dimensionality $d$ and the number of samples $n$ tend to infinity.
This high-dimensional setting is widely studied as it can often -- as in our
experiments -- yield precise predictions for the risk of the estimator when both
the input dimension and the data set size are
large~\cite{Buhlmann11,Wainwright19}. 
It is also the predominant setting considered in previous theoretical works that
discuss overparameterized linear models~\cite{Ali20,Deng20,Dobriban15,
Hastie19,Javanmard20a,Javanmard20b,Sur19}.
%


\subsection{\Natrisk{} and robust risk}



We now introduce the \natrisk{} and robust evaluation metrics for regression and
classification.
Given a pointwise test loss $\losstest \colon \R \times \R \to \R$, we define the \natrisk{} (population) risk of an estimator
$\thetahat$ as
\begin{equation}
    \label{eq:SR}
    \SR(\thetahat) :=  \EE_{X\sim\probx} \losstest\left(\langle \thetahat, X\rangle, \langle
    \thetastar, X\rangle\right),
\end{equation}
where the expectation is taken over the marginal feature distribution $\probx$.
Note that for any data-dependent estimator $\thetahat$,
this risk is fixed if conditioned on the training data. Our asymptotic bounds
hold almost surely over draws of the training set.
As standard in the literature, we choose the square loss $\losstest(u,v)= (u-v)^2$
for regression and the 0-1 loss
$\losstest(u,v) = \indicator_{\sgn(u) \neq \sgn(v)}$
for classification.

The broad application of ML models in
real-world decision-making processes increases requirements on their
generalization abilities beyond \iid{} test sets.
For example, in the image domain, classifiers
should be \emph{robust} and output the same prediction for
perturbations of an image that do not change the ground truth label (e.g., imperceptible $\ell_p$-perturbations~\cite{Goodfellow15}).
In this case, we say the perturbations are \emph{consistent} and the estimator that achieves zero robust population risk
also has zero standard population risk.
For linear models in particular,
one way to enforce consistency is to restrict perturbations to the space
orthogonal to the ground truth, as proposed in~\cite{Raghunathan20}.

Motivated by the imperceptibility assumption and $\ell_p$-adversarial attacks widely studied in the image domain,
we consider the adversarially robust risk of a parameter $\theta$ with respect to consistent $\ell_p$-perturbations
\begin{align}
    \label{eq:AR}
    \AR(\thetahat) \define \EE_{X\sim\probx} \max_{\delta \in \pertset{p}(\epsilon)
    }\losstest(\langle \thetahat, X + \delta\rangle, \langle \thetastar, X\rangle) \:,
\end{align}
with the perturbation set $\pertset{p}(\epsilon) \define \{\delta \in \R^d: \|\delta\|_p \leq \epsilon~\textrm{and} ~\langle \thetatrue,\delta\rangle = 0\}$.

In many scientific applications, security against adversarial attacks may not be
the dominating concern;
one may instead require estimators to be robust against
small distribution shifts. Earlier work~\cite{sinha18} has pointed out that
distribution shift robustness and adversarial robustness are equivalent for
losses that are convex in the parameter $\theta$.
Similarly, in our setting, adversarial robustness against consistent $\ell_p$-perturbations
implies distributional robustness against $\ell_p$-bounded mean shifts in the
covariate distribution $\probx$ (see \suppmat{}~\ref{sec:dro}).

%% file: lin_reg.tex
\section{Min-\texorpdfstring{$\ell_2$}{l2}-norm interpolation in robust linear regression}
\label{sec:linreg}

In the context of regression, we illustrate
overfitting of the robust risk in Equation~\eqref{eq:AR}
with the set of consistent $\ell_2$-perturbations $\pertset{2}(\eps)$.
More precisely, we show that
preventing min-$\ell_2$-norm interpolation on noiseless samples via ridge
regularization improves the robust risk. 
\ifarxiv
Furthermore, we explain why regularization benefits
the robust risk even in settings where it does not decrease the standard risk.
\fi
Lastly, we note that due to the rotational invariance of the problem,
our results hold for sparse and dense ground truths $\thetastar$ alike.

%
%
%
%
%


\subsection{Interpolating and regularized estimator}

We study linear ridge regression estimates defined as
\begin{equation}
    \label{eq:linregridge}
    \thetalambda =
    \arg\min_{\theta} \frac{1}{n} \sum_{i=0}^n (y_i - \langle \theta, x_i \rangle)^2
    + \lambda \|\theta\|_2^2.
\end{equation}
The min-$\ell_2$-norm interpolator  is the limit of the linear ridge regression estimate with
$\lambda \to 0$ and is given by
\begin{equation}
\label{eq:minnorm} \thetaminnorm = \arg\min_{\theta} \|\theta\|_2
~~\mathrm{such~that}~~ \langle \theta, x_i\rangle = \y_i ~~\mathrm{for~all}~ i.
\end{equation}
%
%
Note that the min-$\ell_2$-norm interpolator
is also the estimator that gradient descent on the unregularized loss converges to,
while ridge regression with $\lambda > 0$
corresponds to early-stopped estimators~\cite{Ali19,Ali20}.
Therefore, by proving that a ridge regularized estimator with $\lambda >0$ significantly
outperforms the min-$\ell_2$-norm interpolator with $\lambda \to 0$,
we also show that early stopping benefits robust generalization.

Whenever the goal is to achieve a low robust risk,
a popular alternative to using the standard linear regression
estimate in Equations~\eqref{eq:linregridge},\eqref{eq:minnorm}
is to consider adversarially
trained estimators~\cite{Goodfellow15,Javanmard20a}.
However, $\ell_2$-adversarial training in its usual form (i.e., with inconsistent perturbations)
prevents regression estimators from interpolating, and hence, has a similar effect to
$\ell_2$-regularization as we discuss in more detail
in \suppmat{}~\ref{sec:linreg_inconsistent}.
On the other hand, training with consistent perturbations as defined in the robust risk
is equivalent to full knowledge of the direction of $\thetastar$
and hence simply recovers the ground truth in the noiseless case.
Since our goal is to reveal the shortcomings of interpolators
compared to regularized estimators,
we only analyze ridge estimators trained without perturbations.

\subsection{Robust overfitting in noiseless linear regression}


The following theorem provides a precise asymptotic expression of the
robust risk under consistent $\ell_2$-perturbations
for the ridge regression estimate in Equation~\eqref{eq:linregridge}.
The proof extends
techniques from previous works~\cite{Hastie19,Dobriban15} based on results from random matrix theory~\cite{Bai10,Knowles16}
and can be found in \suppmat~\ref{sec:prooflinregthm}. Without loss of generality, we can assume that $\|\thetastar\|_2 = 1$.

\begin{theorem}
    \label{thm:main_thm_lr}
Assuming the marginal input distribution $\probx = \Normal(0,\idmat_d)$, the robust risk~\eqref{eq:AR}
of the estimator $\thetalambda$ for $\lambda > 0$ (defined
in~\eqref{eq:linregridge}) with respect to
consistent $\ell_2$-perturbations $\pertset{2}(\epsilon)$
asymptotically converges to
\begin{align}
    \AR(\thetalambda) \: \overset{\text{a.s.}}{\longrightarrow}
    \:\Riskinfty + \epstest^2 \Pperpinfty + \sqrt{\frac{8
            \epstest^2}{\pi} \Pperpinfty \Riskinfty} =: \limitrisk
\end{align}
as $d,n \to \infty$ with $d/n \to \gamma$, where $ {\Pperpinfty= \Riskinfty -\lambda^2
(m(-\lambda))^2}$ and ${\Riskinfty = \lambda^2 m'(-\lambda) +\sigma^2 \gamma (m(-\lambda) - \lambda m'(-\lambda))}$
is the asymptotic standard risk, i.e., ${\SR(\thetalambda)\:
  \overset{\text{a.s.}}{\longrightarrow}\: \Riskinfty}$. The function $m(z)$ is the Stieltjes transform of the Marchenko-Pastur law and is given by $m(z) = \frac{1 -
    \gamma - z - \sqrt{(1-\gamma -z)^2 - 4 \gamma z }}{2\gamma z}$.
Further, the limit $\lim_{\lambda \to 0}\limitrisk$ exists for all $\eps \geq 0$
and corresponds to the asymptotic standard ($\eps = 0$) and robust risks ($\eps >0$) of the min-$\ell_2$-norm interpolator $\thetaminnorm$~\eqref{eq:minnorm}.


\end{theorem}

%

We plot the precise asymptotic risks of the ridge estimate with optimal
regularization parameter
$\lambdaopt$\footnote{
    Here we choose $\lambda$ using the population risk oracle.
    In practice, one would resort to standard tools
    such as cross-validation techniques
    that also enjoy theoretical guarantees
    (see e.g.\ \cite{Patil21}).
}
and of the min-$\ell_2$-norm interpolator
for $\lambda \to 0$ in Figure~\ref{fig:main_thm_lr}.
For the robust risk, we use $\epsilon=0.4$.
We first observe in Figure~\ref{fig:lra} that ridge regularization reduces the
robust risk even for $d/n \gg 1$ well beyond the interpolation threshold -- the regime
where previous works show that the variance is negligible, and hence, regularization does not
improve generalization.

Moreover, Figure~\ref{fig:lrb} illustrates that the beneficial effect of ridge
regularization persists even for noiseless data.
This supports our statement that regularization not only helps to reduce
variance,
but also reduces the part of the robust risk that is unaffected by noise in the
overparameterized regime.
Furthermore, we show that experiments run with finite values of $d$ and $n$
(depicted by the markers in Figure~\ref{fig:main_thm_lr}) closely match the
predictions obtained from Theorem~\ref{thm:main_thm_lr} for $d, n \to \infty$
and $d/n \to \gamma$. This indicates that the high-dimensional asymptotic regime
does indeed correctly predict and characterize the high-dimensional
non-asymptotic regime.
Finally, even though Theorem~\ref{thm:main_thm_lr} assumes isotropic
Gaussian covariates, we can extend it to more general
covariance matrices following the same argument as in~\cite{Hastie19}, based on results from random matrix theory~\cite{Knowles16}.

\begin{figure*}
    \centering
    \includegraphics[width=5.5in]{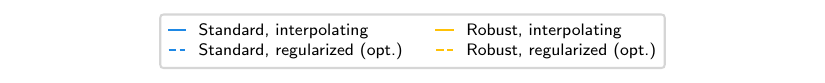}
    \begin{subfigure}{0.49\textwidth}
        \centering
        \includegraphics[width=2.7in]{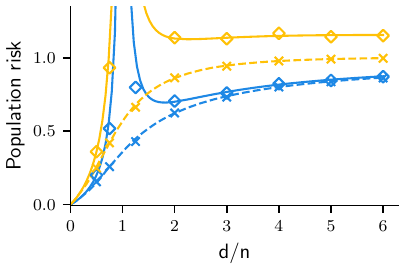}
        \caption{Noisy observations}
        \label{fig:lra}
    \end{subfigure}
    \begin{subfigure}{0.49\textwidth}
        \centering
        \includegraphics[width=2.7in]{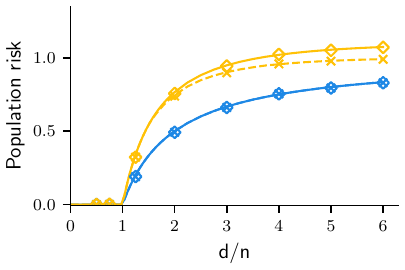}
        \caption{Noiseless observations}
        \label{fig:lrb}
    \end{subfigure}
    \caption{
        Asymptotic theoretical predictions for $d,n\to\infty$ (curves)
        and experimental results with finite $d$ and $n=10^3$ (markers)
        for the robust ($\epstest=0.4$) and standard risk of the
        min-$\ell_2$-norm interpolator
        (solid, \emph{interpolating}) and the ridge
        regression estimate with optimal $\lambda$ (dashed,
        \emph{regularized}) for (a) noisy data with $\sigma^2 = 0.2$
        and (b) noiseless data.
        We observe that the gap between
        the robust risk of the interpolating and optimally regularized
        estimators persists even for noiseless observations.
    }
    \label{fig:main_thm_lr}
    \neuripsvspace{-0.2in}
\end{figure*}

\ifarxiv
\subsection{Intuitive explanation and discussion}

We now shed light on the phenomena revealed by
Theorem~\ref{thm:main_thm_lr} and Figure~\ref{fig:main_thm_lr}.
In particular, we discuss why regularization can
reduce the robust risk even in a noiseless setting
and why the effect is indiscernible for the standard risk.


For this purpose, we examine the robust risk as a function of
$\lambda$, depicted in Figure~\ref{fig:linreg_increasing_fit} for
different overparameterization ratios $\gamma>1$ and $\epsilon=0.4$.
The arrows point in the direction of increasing $\lambda$.
We observe how the minimal robust risk is achieved for a $\lambdaopt$
bounded away from zero and how the optimum increases with the
overparameterization ratio $d/n \to\gamma$, indicating that stronger
regularization is needed the more overparameterized the estimator is.

In order to understand this overfitting phenomenon better, we decompose the
ridge estimate $\thetalambda$ into its projection $\Pi_{\parallel}$ onto the
ground truth direction $\thetatrue$ and the projection $\Pi_\perp$
onto the orthogonal complement,
i.e., $\thetalambda = \Pi_\parallel \thetalambda + \Pi_\perp \thetalambda$.
For the noiseless setting ($\sigma^2=0$), substituting this decomposition into
Equation~\eqref{eq:AR} yields the following closed-form expression of the robust
risk which now involves the parallel error $\|\thetanull - \Pi_{\parallel}
\thetalambda\|_2^2$ and the orthogonal error $\|\Pi_{\perp} \thetalambda\|_2^2$:
\begin{align}
\label{eq:ARnoiseless}
\AR(\thetalambda) =  \|\thetanull - \Pi_{\parallel} \thetalambda\|_2^2  + (1+\epstest^2) \|\Pi_{\perp} \thetalambda\|_2^2
+ \sqrt{\frac{8 \epstest^2 }{\pi} \|\Pi_{\perp} \thetalambda\|_2^2 (\|\thetanull - \Pi_{\parallel} \thetalambda\|_2^2 + \|\Pi_{\perp} \thetalambda\|_2^2 )}.
\end{align}
We provide the proof in \suppmat{}~\ref{sec:linrobrisk}.

Figure~\ref{fig:tradeofflinreg} shows that, as $\lambda$ increases, the
orthogonal error decreases faster than the parallel error increases.
Since, by Equation~\eqref{eq:ARnoiseless}, the orthogonal error is weighted more
heavily for large enough perturbation strengths $\epstest$, some nonzero ridge
coefficient
yields the best trade-off.
On the other hand, the standard risk with $\epstrain =0 $
weighs both errors equally, resulting in an optimum at $\lambda \to 0$.

\begin{figure*}
	\centering
	\includegraphics[width=5.5in]{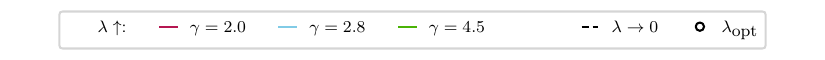}
	\begin{subfigure}{0.61\textwidth}
		\centering
		\includegraphics[width=3.15in]{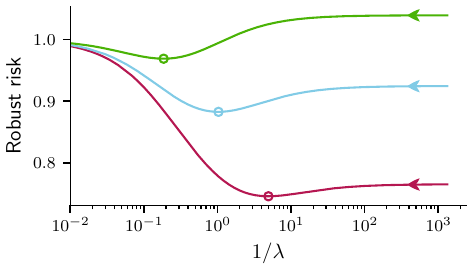}
		\caption{Robust risk as $\lambda$ increases}
		\label{fig:linreg_increasing_fit}
	\end{subfigure}
	\begin{subfigure}{0.38\textwidth}
		\centering
		\includegraphics[width=2.25in]{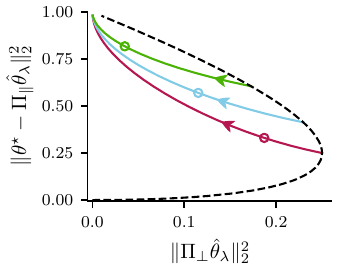}
		\caption{Orthogonal and parallel parameter error}
		\label{fig:tradeofflinreg}
	\end{subfigure}
	\caption{
		Theoretical curves depicting the robust risk with $\epstest=0.4$ (a)
		and decomposed terms (b) as $\lambda$ increases (arrow direction)
		for different choices of the overparameterization ratio $d/n\to\gamma$.
		In (b) we observe that for large $\gamma > 1$, as $\lambda$ increases,
		the orthogonal error $\|\Proj_\perp \thetalambda\|_2$ decreases
		whereas the parallel error $\|\thetatrue - \Proj_\parallel \thetalambda\|_2$ increases.
		For $\eps>0$, the optimal $\lambda$ is large enough to prevent interpolation.
	}
	

\label{fig:linreg_noiseless_theory}
\end{figure*}

\fi

%% file: log_reg.tex
\section{Max-\texorpdfstring{$\ell_2$}{l2}-margin interpolation in robust linear classification}
\label{sec:logreg}

%


Unlike linear regression, adversarially trained
binary logistic regression classifiers
may still interpolate the training data,
resulting in
\emph{robust} max-$\ltwo$-margin interpolators as $\lambda \to 0$.
Hence, in this section we train and evaluate classifiers
with $\ell_\infty$-perturbation sets
$\pertset{\infty}(\eps)$, a standard choice in the experimental and theoretical
classification literature \cite{Goodfellow15,Javanmard20b,Rice20,Sanyal21},
but also discuss $\ell_2$-perturbations in \suppmat~\ref{sec:logreg_l2_perturbations} for completeness.
Our theoretical results show that the robust max-$\ltwo$-margin interpolator with $\lambda \to 0$
has a worse robust risk than a regularized predictor with $\lambda >0$.

\subsection{Interpolating and regularized estimator} \label{sec:estimator}

As discussed in Section~\ref{sec:linreg}, a common method to obtain robust
estimators is to use adversarial training. However, for linear regression,
adversarial training either renders interpolating estimators infeasible, or
requires oracle knowledge of the ground truth.
In contrast, for linear classification, interpolation is easier to achieve -- it
only requires the sign of $\langle x_i , \theta\rangle$ to be the same as the
label $y_i$ for all $i$.  In particular, when the data is sufficiently high
dimensional,  it is possible to find an interpolator of the adversarially
perturbed training set.

We study the robust ridge-regularized
logistic regression estimator with penalty weight $\lambda > 0$,
\begin{equation}
  \label{eq:logridge}
  \thetalambda \define \arg\min_\theta \frac{1}{n}\sum_{i=1}^n \max_{\delta \in \pertset{\infty}(\epstrain)}\log(1 + e^{-\langle \theta, x_i+\delta\rangle y_i}) + \lambda \|\theta\|_2^2.
\end{equation}
In the limit $\lambda \to 0$ the results in~\cite{Rosset04} imply that  the robust ridge-regularized logistic
regression estimator from Equation~\eqref{eq:logridge}
directionally aligns with the
robust
max-$\ell_2$-margin interpolator:\footnote{While~\cite{Rosset04} only proves the result
for $\eps = 0$, it is straightforward to extend it to the general case
where $\eps \geq 0$.}
\begin{align}
  \label{eq:maxmarginAE}
  \thetaminnorm \define &\arg\min_{\theta}
  \|\theta\|_2 ~~\mathrm{such~that} \min_{\delta\in
  \pertset{\infty}(\eps)} \y_i \langle \theta, x_i + \delta \rangle \geq 1 ~\mathrm{for~all}~ i.
\end{align}

We say that the data is \emph{robustly separable} if the robust max-$\ell_2$-margin interpolator exists.

The robust max-$\ell_2$-margin solution is an interpolating
 estimator of particular importance
 since
it directionally aligns with the estimator found by gradient descent~\cite{Li2020}.
Since the robust accuracy (i.e., the robust risk defined using the 0-1 loss)
is independent of the norm of the estimator, we simply refer to the robust
max-$\ell_2$-margin solution as the normalized vector
$\frac{\thetaminnorm}{\|\thetaminnorm\|_2}$.

In this paper, we study two
choices for the set of training perturbations $\pertset{\infty}(\epsilon)$:
\begin{align}
 &\text{inconsistent perturbations} &\pertset{\infty}(\epsilon) = \{\delta \in \R^d: \|\delta\|_\infty \leq
\epsilon\} \label{eq:inconsistentpert}\\
&\text{consistent perturbations}  &\pertset{\infty}(\epsilon)  =  \{\delta \in \R^d: \|\delta\|_\infty  \leq
\epsilon , \langle \delta, \thetastar \rangle  = 0\} \label{eq:consistentpert}
\end{align}

Adversarial training with respect to inconsistent
perturbations \eqref{eq:inconsistentpert} is a popular choice in the
literature to improve the robust risk (e.g.\ \cite{Goodfellow15,
Javanmard20b}).  However, perturbed samples
may cross the true decision boundary, and hence, inconsistent
perturbations effectively introduce noise during the training procedure.
In particular, in the data model with noiseless observations that we introduce in
Section~\ref{sec:mainlogreg}, the ground truth function misclassifies
approximately $8\%$ of the labels when perturbing the training data with
inconsistent perturbations of size $\epsilon = 0.1$.

As mentioned in the introduction, in this paper we are interested in
verifying whether regularization can be beneficial in high dimensions
even in the absence of noise. 
Therefore, in the sequel we study the impact of
both inconsistent~\eqref{eq:inconsistentpert}
and consistent~\eqref{eq:consistentpert}
perturbations on robust overfitting.

\begin{figure*}
	\centering
	\arxivvspace{0.3em}
	\begin{subfigure}[t]{0.62\textwidth}
		\centering
		\includegraphics[width=3.6in]{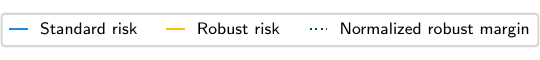}
		\arxivvspace{0.7em}
	\end{subfigure}
	\begin{subfigure}[l]{0.49\textwidth}
		\centering
		\includegraphics[width=2.6in]{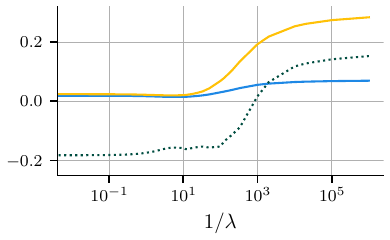}
		\caption{Benefit of ridge regularization}
		\label{fig:logreg_explanation_standard_and_robust_risk}
	\end{subfigure}
	\begin{subfigure}[r]{0.49\textwidth}
		\centering
		\includegraphics[width=2.6in]{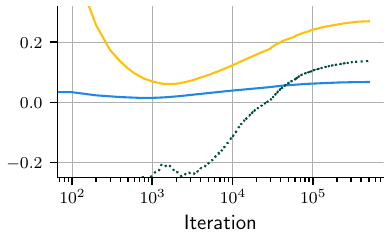}
		\caption{Benefit of early stopping}
		\label{fig:early_stopping}
	\end{subfigure}
	\arxivvspace{0.2em}
	\caption{
		Normalized robust margins and risks
		of empirical simulations using $\eps =0.1$ and $d/n=8$,
		with respect to
		(a) increasing $1/\lambda$ and
		(b) gradient descent iterations when minimizing
		Equation~\eqref{eq:logridge} using $\lambda=0$.
		Both ridge regularization and early stopping
		yield superior robust and standard risks.
		Each experiment uses $n = 10^3$
		and inconsistent $\ell_\infty$-perturbations for training.
		See Appendix~\ref{sec:appendix_experiments} for more details.}
	\label{fig:logreg_empirical}
	\neuripsvspace{-0.15in}
\end{figure*}

\subsection{Robust overfitting in noiseless linear classification}
\label{sec:mainlogreg}
We now show empirically that regularization helps to improve the
robust and standard risks when training with noiseless data and derive precise
asymptotic predictions for both risks.
Throughout this subsection we assume
deterministic, and hence, noiseless training labels, i.e.,~$y_i = \sgn \langle
x_i, \thetastar \rangle$. Furthermore, as we discuss in Section~\ref{sec:logreg_discussion},
the inductive bias of the $\ell_\infty$-robust logistic loss encourages sparse
solutions.
Since we are primarily interested in learning ground truth functions
that match the implicit bias of the estimator,
we assume the sparse ground truth $\thetastar = (1,0,\ldots,0)^T$.

We first show robust overfitting experimentally on noiseless data
when training with inconsistent perturbations and subsequently demonstrate that overfitting persists
even if the training procedure is completely noiseless
(i.e., using consistent training perturbations).
Finally, we provide theoretical evidence for our observations in the high-dimensional asymptotic limit.

%
%
%
%

%
%
%

\paragraph{Training with inconsistent adversarial perturbations}
Figure~\ref{fig:logreg_explanation_standard_and_robust_risk} illustrates the
\emph{robust margin} \\$\min_{i}{ \min_{\delta\in
\pertset{\infty}(\eps)}{\frac{1}{\|\theta\|_2} \y_i \langle \theta, x_i + \delta \rangle }
} $
as well as the standard and robust risks of the estimator
$\thetalambda$ trained using inconsistent adversarial perturbations
on a synthetic data set
with fixed overparameterization ratio $d/n = 8$.
We observe that decreasing the ridge coefficient well beyond
the point where the minimizer of the robust logistic loss \eqref{eq:logridge}
reaches $100\%$ robust training accuracy (i.e., the robust margin becomes
positive) substantially hurts generalization.



In addition to varying the ridge coefficient $\lambda$, we notice that the same
trends as for $\lambda \to 0$ 
also occur for the gradient descent
optimization path as the number of iterations $t$ goes to infinity.
Figure~\ref{fig:early_stopping} indicates that, similarly to ridge
regularization, early stopping also avoids the robust max-$\ell_2$-margin solution that is
obtained for $t \to \infty$
and yields an estimator with significantly lower
standard and robust risks.

\arxivvspace{0.4em}
\paragraph{Training with consistent adversarial perturbations}
As discussed in Section~\ref{sec:estimator},
even for noiseless training data,
inconsistent perturbations can induce noise
during the training procedure.
Hence, one could hypothesize that the
noise induced by the inconsistent perturbations causes the overfitting observed
in Figure~\ref{fig:logreg_empirical}.
To contradict this hypothesis, we also
study adversarial training with consistent perturbations. By definition,
consistent perturbations do not cross the true decision boundary
and hence leave the training data entirely noiseless.

Figure~\ref{fig:logreg_inconsistent_vs_consistent}
shows that the adversarially trained estimators \eqref{eq:logridge},\eqref{eq:maxmarginAE}
with consistent and inconsistent perturbations
yield comparable robust risks. Moreover, robust overfitting occurs in both
situations,
as the risk is higher for the interpolating estimator
$\lambda\to 0$ compared to an optimal $\lambda >0$.
Hence, our observations demonstrate that robust overfitting persists
even if training with consistent perturbations in an entirely noiseless setting.
This observation is counter intuitive since,
according to classical wisdom,
we would expect ridge regularization to only benefit
in noisy settings where the estimator suffers from a high variance.


We now prove this phenomenon using the next theorem.
In particular, similar
to Theorem~\ref{thm:main_thm_lr} for linear regression,
we show that
robust overfitting occurs in the high-dimensional asymptotic regime where $d/n \to \gamma$ as $d,n\to\infty$.
We state an informal version of the theorem in the main text
and refer to \suppmat{}~\ref{sec:logregtheory} for the precise statement. The
proof is inspired by the works \cite{Javanmard20b,Salehi19} and uses
the Convex Gaussian Minimax Theorem~(CGMT)~\cite{Gordon88,Thrampoulidis15}.

\begin{theorem}[Informal]
  \label{theo:logreg}
Assume that $\eps = \epsnull/\sqrt{d}$ for some constant $\epsnull$  and
${\thetastar = (1,0,\cdots,0)^T}$.
Then, the robust and standard risks of the regularized estimator~$\thetalambda$~\eqref{eq:logridge} ($\lambda > 0$)
and of the robust max-$\ell_2$-margin interpolator~\eqref{eq:maxmarginAE}~($\lambda \to 0)$ 
with inconsistent \eqref{eq:inconsistentpert} or consistent \eqref{eq:consistentpert}
adversarial $\ell_\infty$-perturbations
converge in probability as $d,n \to \infty$, $d/n \to \gamma$ to:
\neuripsvspace{-0.1in}
 \begin{align*}
 \SR(\thetalambda) &\to \frac{1}{\pi}\arccos\left(\frac{\nupar^\star}{\nu^\star}\right) ~~\mathrm{and}~~
  \AR(\thetalambda) \to \SR(\thetalambda) +  \frac{1}{2} \erf\left(\frac{\epsnull \delta^\star}{\sqrt{2} \nu^\star}\right) + I\left(\frac{\epsnull \delta^\star}{\nu^\star}, \frac{\nupar^\star}{\nu^\star}\right)
\end{align*}
We denote by $\erf(.)$ the error function,
\begin{equation*}
 I(t,u) \define \int_{0}^{t} \frac{1}{\sqrt{2\pi}}
 \exp\left(-\frac{x^2}{2}\right)\erf\left(\frac{x u}{\sqrt{2(1- u^2)}}\right)dx,
\end{equation*}
and use the notation $\nu^\star = \sqrt{(\nuper^\star)^2 + (\nupar^\star)^2}$,
where $\nuper^\star,\nupar^\star,\delta^\star$ are the unique solution of a scalar optimization problem specified in Appendix~\ref{sec:logregtheory} that depends on $\thetatrue, \gamma, \epsnull$ and $\lambda$.
\end{theorem}

Since the theoretical expressions are hard to interpret,
we visualize the asymptotic values of the standard and robust risks
from Theorem~\ref{theo:logreg}
in Figure~\ref{fig:logreg_theory_combined}
by solving the scalar optimization problem
specified in Appendix~\ref{sec:logregtheory}.
We observe that Theorem~\ref{theo:logreg} indeed predicts the benefits of
regularization for robust logistic regression and that simulations using finite values of $d$ and $n$
follow the asymptotic trend.
We describe the full empirical setup in Appendix~\ref{sec:appendix_experiments}.

\begin{figure*}
	\centering
	\begin{subfigure}{0.49\textwidth}
		\centering
		\includegraphics[width=2.6in]{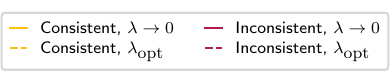}
		\includegraphics[width=2.6in]{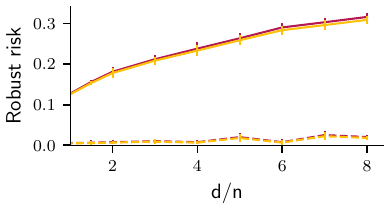}
		\caption{Fixed $\epstrain = 0.1$ for increasing $d$}
		\label{fig:logreg_inconsistent_vs_consistent}
	\end{subfigure}
	\begin{subfigure}{0.49\textwidth}
		\centering
		\includegraphics[width=2.6in]{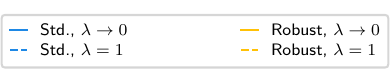}
		\includegraphics[width=2.6in]{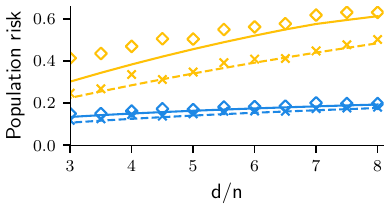}
		\caption{Fixed $n,d$ and asymptotic predictions}
		\label{fig:logreg_theory_combined}
	\end{subfigure}
	
    \caption{ (a) Comparison of consistent and inconsistent
      $\ell_\infty$-perturbations for adversarial logistic regression  with
      respect to the degree of overparameterization $d/n$, using $\epstrain =
      0.1$ for both training and evaluation.  Note that both estimators behave
      very similarly, implying that the effect of inconsistency is negligible
      for small $\eps$.  (b) Robust and standard risks of the robust
      max-$\ell_2$-margin interpolator ($\lambda \to 0$) and robust ridge
      estimate ($\lambda=1$) with consistent perturbations
      \eqref{eq:consistentpert} using $\epstrain = 0.05$ as a function of the
      overparameterization ratio $d/n$ for simulations (markers) and asymptotic
      theoretical predictions from Theorem~\ref{theo:logreg} (lines). We note
      that, for small values of $\gamma$, solving the optimization problem that
      gives the theoretical predictions becomes numerically unstable. All
      simulations use $n=10^3$ samples from our data model; see \suppmat{}
      \ref{sec:appendix_experiments} for further experimental details.  
}

	\label{fig:logreg_theory}
	\neuripsvspace{-0.1in}
\end{figure*}

\subsection{Intuitive explanation and discussion}
\label{sec:logreg_discussion}

Even though we explicitly derive the precise asymptotic expressions of the
standard and robust risks in Theorem~\ref{theo:logreg} that predict the benefits
of regularization for generalization, it is difficult to extract intuitive
explanations for this phenomenon directly from the proof.
We conjecture that a
non-zero ridge penalty induces a more sparse $\thetalambda$ (i.e., with a smaller $\ell_1$/$\ell_2$-norm ratio) than the robust
max-$\ell_2$-margin solution $\thetamaxmarg$ and use simulations to support our
claim. Since the $\linf$-adversarially robust risk penalizes dense solutions
with large ratio of the $\lone/\ell_2$-norms (see Lemma~\ref{lm:logregrisk} in
\suppmat{}~\ref{sec:logrobrisk}), we expect more sparse estimators
to have a lower robust risk. Indeed, Figure \ref{fig:logreg_explanation_sparsity} shows that the $\ell_1/\ell_2$-norm ratio strongly correlates with the robust risk of the estimator.




We begin by noting that, due to Lagrangian duality, minimizing the ridge-penalized loss~\eqref{eq:logridge}
corresponds to minimizing the unregularized loss constrained to the set of
estimators $\theta$ with a bounded $\ell_2$-norm. This norm decreases as the ridge coefficient $\lambda$ increases.
In what follows, we provide intuition on the effect of the ridge penalty $\lambda$ on the
sparsity of the estimator $\thetalambda$.

\vspace{-0.2cm}
\paragraph{Large $\lambda$ inducing a small $\ell_2$-norm}
We first analyze the regularized estimator $\thetalambda$ for large $\lambda$ that constrains solutions to
have small $\ell_2$-norm.
We can therefore use Taylor's theorem and the closed-form expression of
adversarial perturbations (see Lemma~\ref{lm:logregrisk} in
\suppmat{}~\ref{sec:logrobrisk}) to approximate the
unregularized robust loss from Equation~\eqref{eq:logridge} as follows:

\begin{equation}
\label{eq:Taylorexp}
\frac{1}{n} \sum_{i=1}^n \log (1 + e^{-y_i \langle \theta,x_i \rangle + \epsilon \|\Piper \theta\|_1})  \approx  \frac{1}{n}\sum_{i=1}^n -\y_i
\langle x_i, \theta\rangle + \epstrain \| \Piper \theta\|_1.
\end{equation}

As a consequence, the minimizer $\thetalambda$ should result in a large
\emph{robust average margin}
solution, that is, a solution with large $\frac{1}{n } \sum_{i=1}^n
\frac{1}{\|\theta\|_2}\left(y_i \langle x_i, \theta \rangle - \eps \| \Piper
\theta\|_1 \right)$.  Indeed, we observe this using simulations for finite $d,
n$ in
Figures~\ref{fig:logreg_explanation_margins}~and~\ref{fig:logreg_explanation_robust_margin_decay}.
In particular, the objective in Equation~\eqref{eq:Taylorexp} leads to a
trade-off between the sparsity of the estimator (via its convex surrogate, the $\ell_1$-norm) and an \emph{average} of the sample-wise margins $y_i \langle x_i,
\theta \rangle $.  We note that such estimators
have been well studied in the literature and are known to achieve good
performance in recovering sparse ground truths~\cite{Bartlett98,Goa13,Plan12}.

\vspace{-0.2cm}
\paragraph{Small $\lambda$ inducing a large $\ell_2$-norm}
In contrast, a small ridge coefficient $\lambda$ leads to estimators
$\thetalambda$ with large $\ell_2$-norms.  In this case, the estimator
approaches a large \emph{robust (minimum) margin} solution, i.e., a solution
with large $\min_{i}{\frac{1}{\|\theta\|_2} \left( y_i \langle x_i, \theta
\rangle - \eps \| \Piper \theta\|_1 \right)}$ which is maximized by the robust
max-$\ell_2$-margin interpolator~\eqref{eq:maxmarginAE}.  As a consequence, this
leads to a trade-off between estimator sparsity and the robust minimum margin
$\min_i \frac{1}{\|\theta\|_2} y_i \langle x_i, \theta\rangle$.  Due to the high
dimensionality of the input data, the training samples $x_i$ are approximately
orthogonal.  Thus, to achieve a non-vanishing robust margin, estimators are
forced to trade-off sparsity with \emph{all} sample-wise margins instead of just
the average.
We reveal this trade-off in
Figures~\ref{fig:logreg_explanation_sparsity}~and~\ref{fig:logreg_explanation_margins}
where the increase in $\ell_1/\ell_2$-norm ratios corresponds to a decrease in
the robust average margin and an increase in the robust margin.
\\ \\ Finally, we observe that the \emph{sparse} ground truth is characterized
by a large robust average margin (horizontal dotted line in
Figure~\ref{fig:logreg_explanation_margins}) and a small (minimum) robust margin
(horizontal dashed line). Therefore, we expect that the solution that is sparser
and which satisfies the same properties for the robust margin as the ground
truth $\thetastar$, will achieve lower robust and standard risks. Indeed, our
findings indicate that the regularized estimator $\thetalambda$ for large
$\lambda$ aligns better with $\thetastar$, compared to the solution obtained for
a small $\lambda$,
and hence justify the better performance of ridge-regularized
predictors.

\begin{figure*}
	\begin{subfigure}[b]{0.32\textwidth}
		\centering
		\includegraphics[width=1.6in]{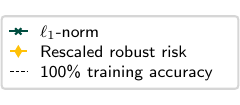}
		\includegraphics[width=1.8in]{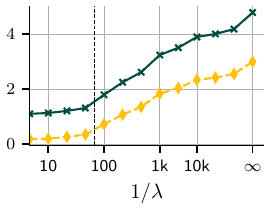}
		\caption{Sparsity of $\thetalambda$}
		\label{fig:logreg_explanation_sparsity}
	\end{subfigure}
	\begin{subfigure}[b]{0.32\textwidth}
		\centering
		\includegraphics[width=1.6in]{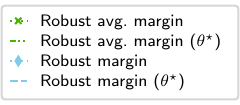}
		\includegraphics[width=1.8in]{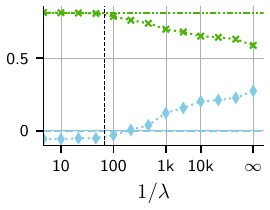}
		\caption{Robust margin of $\thetalambda$}
		\label{fig:logreg_explanation_margins}
	\end{subfigure}
	\begin{subfigure}[b]{0.32\textwidth}
		\centering
		\includegraphics[width=1.6in]{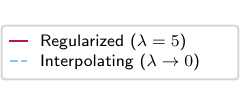}
		\includegraphics[width=1.8in]{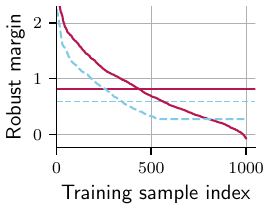}
		\caption{Robust margin decay of $\thetalambda$}
		\label{fig:logreg_explanation_robust_margin_decay}
	\end{subfigure}
	\caption{
		(a) The $\ell_1$-norm and the rescaled (by a factor of 10) robust risk of the estimator with respect to
		$1/\lambda$.
		(b) The robust average
		margin contrasted to the robust margin as a function of $1/\lambda$.
		The horizontal lines denote the corresponding values for
		$\thetatrue$.
		(c) The ordered sample-wise robust margins $y_i \langle
		x_i, \theta\rangle - \epsilon\|\Piper\theta\|_1$ when
		interpolating and regularizing.  For larger $\lambda$, the
		robust (\emph{minimum}) margin decreases while the robust average margin (horizontal
		lines) increases.
		We normalize the estimators, i.e., $\|\thetalambda\|_2 = 1$, for all curves presented in the plots;
		see \suppmat{} \ref{sec:appendix_experiments} for further
		experimental details.
	}
	\label{fig:logreg_explanation}
	\neuripsvspace{-0.2in}
\end{figure*}

\neuripsvspace{-0.2cm}
\subsection{Benefits of an unorthodox way to avoid the robust
max-\texorpdfstring{$\ell_2$}{l2}-margin interpolator}
\label{sec:non_separable_data}
In the previous subsections we focused on robustly separable data and studied
the generalization performance of regularized estimators that do not maximize
the robust margin.
Another way to avoid the robust max-$\ell_2$-margin solution is to introduce
enough label noise in the training data. We now show that, unexpectedly, this
unorthodox way to avoid the robust max-margin solution can also yield an
estimator with better robust generalization than the robust max-$\ell_2$-margin
solution of the corresponding noiseless problem.

Specifically, in our experiments we introduce noise
by flipping the labels of a fixed fraction of the training data.
Figure~\ref{fig:logreg_nonseparable} shows the robust and standard
risks together with the training loss of the estimator $\thetalambda$
from Equation~\eqref{eq:logridge} trained with consistent
perturbations for $\lambda \to 0$ with varying fractions of flipped
labels. For low noise levels, the data is robustly separable and the
training loss vanishes at convergence, yielding the robust
max-$\ell_2$-margin solution in Equation~\eqref{eq:maxmarginAE}.  For
high enough noise levels, the constraints in
Equation~\eqref{eq:maxmarginAE} become infeasible and the training
loss of the resulting estimator starts to increase.  As discussed in
Subsection~\ref{sec:logreg_discussion}, this estimator has a better
implicit bias than the robust max-$\ell_2$-margin interpolator and
hence achieves a lower robust risk.

Even though it is well known that introducing covariate noise can induce implicit
regularization~\cite{Bishop95}, our observations show that
in contrast to common intuition,
the robust risk also decreases when introducing wrong labels in
the training loss. In parallel to our work, the paper \cite{Lee21} shows that
training with corrupted labels can be beneficial for the standard risk.

However, we emphasize here that we do not advocate in favor of artificial label
noise as a means to obtain more robust classifiers. In particular, even if the
data is not robustly separable, the estimator with optimal ridge parameter
$\lambdaopt$ in Figure~\ref{fig:logreg_nonseparable} still always outperforms
the unregularized solution.


\begin{figure*}
	\centering
	\begin{subfigure}[t]{\textwidth}
		\centering
		\includegraphics[width=3.6in]{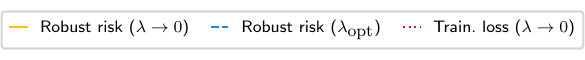}
	\end{subfigure}
	\begin{subfigure}[b]{\textwidth}
		\centering
		\includegraphics[width=3.6in]{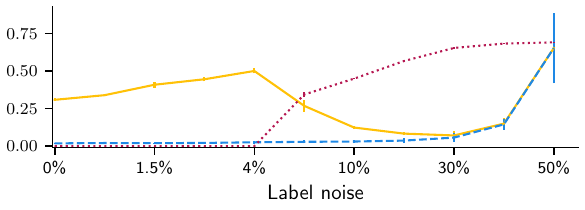}
	\end{subfigure}
	
	\caption{
		Training loss and robust risks with respect to increasing
		training label noise for $\eps=0.1$, $d=8\times 10^3$ and $n=10^3$. We observe for
		unregularized estimators ($\lambda \to 0$) that, counterintuitively,
		moderate amounts of label noise decrease the robust risk by avoiding the
		robust max-$\ell_2$-margin solution.  While this might spuriously imply that
		injecting label noise increases robustness, estimators with optimal ridge
		parameter $\lambdaopt$ still outperform their unregularized counterparts
		in terms of robust risk.
		Since the setting is noisy, we average the risks over five
		independent dataset draws and indicate standard deviations
		via error bars.
	}
	\label{fig:logreg_nonseparable}
	\neuripsvspace{-0.1in}
\end{figure*}

%
%


Finally, we remark that a similar effect can also be observed when training with
inconsistent perturbations
with large perturbation norm $\epsilon$.
We refer to \suppmat{}~\ref{sec:inconsistent} for further discussion.

%% file: related_work.tex
\neuripsvspace{-0.1cm}
\section{Related work}
%

\neuripsvspace{-0.2cm}
\paragraph{Understanding robust overfitting}
The current literature attempting to explain
robust overfitting~\cite{Rice20}
primarily focuses on the effect of noise and on the smoothness of decision boundaries
learned by neural networks trained to convergence~\cite{Dong21,Sanyal21,Wu21}.
A slightly different line of work~\cite{Sagawa20a} shows that
overparameterized models require regularization
in order to achieve good classification accuracy on
all subpopulations.
However, a theoretical understanding of the role of regularization is still missing.

\neuripsvspace{-0.2cm}
\paragraph{Theory for adversarial robustness of linear models}
Recent works~\cite{Javanmard20a,Javanmard20b}
provide a precise description of the robust risks for logistic and linear regression when trained with adversarial attacks based on the Convex Gaussian Minimax Theorem~\cite{Gordon88,Thrampoulidis15}.
The analysis focuses on inconsistent attacks for
both training and evaluation.
For linear regression, the authors observe that adversarial
training with $\ell_2$-perturbations mitigates
the peak in the double descent
curve around $d = n$ and hence acts similarly to ridge regularization
as explicitly studied in~\cite{Hastie19}.
Several other works focus on the role of gradient descent
in robust logistic regression.
In particular, \cite{Zou21}~proves that early stopping
yields robust adversarially-trained linear classifiers
even in the presence of noise.
Furthermore, the results of~\cite{Li2020} show that
gradient descent on robustly separable data converges to
the robust max-$\ell_2$-margin estimator~\eqref{eq:maxmarginAE}.

%% file: conclusion.tex
\section{Conclusion and future work}


In this work, we show that overparameterized linear models can overfit with
respect to the robust risk even when there is no noise in the training data.
Our results challenge the modern narrative that interpolating
overparameterized models yield good performance
without explicit regularization
and motivate the use of ridge regularization and early stopping for improved
robust generalization.
Perhaps surprisingly, we further observe that ridge
regularization enhances the bias of logistic regression trained with adversarial
$\ell_{\infty}$-attacks towards sparser solutions,
indicating that the impact of
explicit regularization may go well beyond variance reduction.

\neuripsvspace{-0.2cm}
\paragraph{Future work}
Our simulations indicate that early stopping yields similar benefits
as ridge regularization in noiseless settings.
However, we leave a formal proof for future work.
Furthermore, the double descent phenomenon has been proven for
the standard risk on a broad variety of data distributions
and even for non-linear models such as random feature regression.
It is still unclear how our results translate to these settings.
In particular, our theoretical analysis
heavily makes use of the closed-form solution of the adversarial attacks, and
hence, cannot be applied to non-linear models in a straightforward way.

%
%
%

\section*{Acknowledgments}

K.\ D.\ is supported by the ETH AI Center and the ETH Foundations of Data Science.
R.\ H.\ is supported by the IAS at TUM, the DFG (German Research Foundation), and by the NSF under award IIS-1816986.

%% file: appendix/setting.tex
\section{Closed form expressions for the robust risks}
%

In Section \ref{sec:linrobrisk} and \ref{sec:logrobrisk} we
derive closed-form expressions of the standard and robust risks from Equations~\eqref{eq:SR},\eqref{eq:AR}
for the settings studied in Section~\ref{sec:linreg},\ref{sec:logreg}.
We use those expressions repeatedly in our proofs.
Furthermore, Section~\ref{sec:dro} discusses that the robust risk \eqref{eq:AR}
upper-bounds the worst case risk under distributional mean shifts.

\subsection{Closed-form of robust risk for regression}
\label{sec:linrobrisk}
The following lemma provides a closed-form expression
of the robust risk for the linear regression setting
studied in Section \ref{sec:linreg}.
A similar result for inconsistent attacks
has already been shown before (Lemma~3.1. in~\cite{Javanmard20a});
we only include the proof for completeness.

\begin{lemma}
\label{lm:linregrisk}
Assume that $\prob_X$ is the isotropic Gaussian distribution. Then, for the square loss, the robust risk \eqref{eq:AR} with respect to $\ell_2$-perturbations is given by
\begin{equation}
\label{eq:ARreg}
 \AR(\theta) = \| \thetatrue - \theta\|_2^2 + 2\epsilon \sqrt{2/\pi} \| \Piper \theta\|_2 \| \thetatrue - \theta \|_2 + \epsilon^2 \|\Pi_{\perp} \theta \|^2_2.
\end{equation}
\end{lemma}
%
%
\begin{proof}
Define
$\tilde y_i = y_i - \innerprod{x_i}{\theta}$, and
note that using similar arguments as in Section 6.2. \cite{Javanmard20a}
\begin{align*}
\max_{ \delta_i \in \pertset{2}(\epsilon) } ( \tilde y_i - \innerprod{\delta_i}{\theta} )^2
&=
 ( \max_{\delta_i \in \pertset{2}(\epsilon)} \left| \tilde y_i - \innerprod{\delta_i}{\theta} \right|)^2 \\
&=
 ( \left| \tilde y_i \right|  + \max_{\norm[2]{\delta_i} \leq \epsilon, \delta_i \perp \thetastar} \left| \innerprod{\delta_i}{\theta} \right|)^2 \\
&=
( \left| \tilde y_i \right|   + \epsilon \norm[2]{\Pi_{\perp} \theta} )^2.
\end{align*}
With this characterization, we can derive a convenient expression for the robust risk:
\begin{align}
\AR(\theta)
&=
\EE_{X}{ ( \left| \langle X, \thetastar  - \theta \rangle \right| + \epsilon \norm[q]{ \Pi_{\perp} \theta} )^2 } \nonumber \\
&=
\EE_{X }{ \left( \langle X, \thetastar  - \theta \rangle  \right)^2 }
+
2\epsilon \EE_{X }{ \left|\langle X, \thetastar  - \theta \rangle \right| }  \norm[2]{\Pi_{\perp}\theta}
+
\epsilon^2 \norm[2]{\Pi_{\perp}\theta}^2.
\label{eq:robustriskexpressiongeneral}
\end{align}
Since we assume isotropic Gaussian features,
that is $\prob_X = \Normal(0,\idmat)$,
we can further simplify
%
%
%
\begin{align*}
\AR(\theta)
&=
\norm[2]{\theta - \theta^\ast}^2
+
2\epsilon
\sqrt{2/\pi}
 \norm[2]{\Pi_{\perp}}
\norm[2]{\theta - \theta^\ast}
+
\epsilon^2 \norm[2]{\Pi_{\perp}}^2
\end{align*}
which concludes the proof.
\end{proof}

%
%

\subsection{Closed-form of robust risk for classification}
\label{sec:logrobrisk}
Similarly to linear regression, we can express the robust and standard risk
for the linear classification model in Section~\ref{sec:logreg} as
stated in the following lemma.
\begin{lemma}
\label{lm:logregrisk}
Assume that $\prob_X$ is the isotropic Gaussian distribution and
$\thetastar = (1,0,\cdots,0)^\top$. Then,
\begin{enumerate}
  \item For any non-decreasing loss $\lossf: \R \to \R$ we have
\begin{equation}
\label{eq:logregclosedform}
 \max_{ \delta_i \in \pertset{\infty}(\epsilon) }  \lossf( y_i\innerprod{x_i + \delta_i}{\theta} ) = \lossf( y_i\innerprod{x_i }{\theta} - \eps \|\Piper \theta\|_1  ).
\end{equation}

\item For the $0$-$1$ loss the
robust risk \eqref{eq:AR} with respect to $\ell_{\infty}$-perturbations is given by
\begin{align}
  \label{eq:robriskclass}
  \AR(\theta) = \frac{1}{\pi}\arccos\left(\frac{\langle\thetanull, \theta\rangle}{\|\theta\|_2}\right) + \frac{1}{2} \erf \left(\frac{\eps\|\Piper \theta\|_1}{\sqrt{2} \|\theta\|_2}\right) + I\left(\frac{\eps \|\Piper \theta\|_1}{\|\theta\|_2}, \frac{\langle\thetanull,\theta\rangle}{\|\theta\|_2}\right),
\end{align}
with
\begin{equation}
 I(t,u) \define \int_{0}^{t} \frac{1}{\sqrt{2\pi}} \exp\left(-\frac{x^2}{2}\right)\Phi\left(\frac{x u}{\sqrt{2}\sqrt{1- u^2}}\right)dx.
\end{equation}
\end{enumerate}
\end{lemma}
\begin{proof}

We first prove Equation~\eqref{eq:logregclosedform}.
Because $\lossf$ is non-increasing, we have
\begin{align*}
\max_{ \delta_i \in \pertset{\infty}(\epsilon) }  \lossf( y_i\innerprod{x_i + \delta_i}{\theta} )
&=
 \lossf( \min_{\delta_i \in \pertset{\infty}(\epsilon)} y_i\innerprod{x_i + \delta_i}{\theta} ) \\
&=
 \lossf( y_i\innerprod{x_i }{\theta}  + \min_{\norm[\infty]{\delta_i} \leq \epsilon, \delta_i \perp \thetastar} \innerprod{\delta_i}{\theta}  ) \\
&=
 \lossf( y_i\innerprod{x_i }{\theta}  - \eps \|\Piper \theta\|_1  ),
\end{align*}
which establishes Equation~\eqref{eq:logregclosedform}. Note that for the last equation we used that, while minimization over $\delta$ has no closed-form solution in general, for our choice of $\thetastar = (1,0,\cdots,0)$, we get $\min_{\norm[\infty]{\delta_i} \leq \epsilon, \delta_i \perp \thetastar} \innerprod{\delta_i}{\theta}  = - \eps \|\Piper \theta\|_1$.


We now prove the second part of the statement.
Let $\indicator\{E\}$ be the indicator function, which is $1$ if the event $E$ occurs, and $0$ otherwise.
Since $\lossf(\cdot)=\indicator_{\cdot\leq 0}$ is non-increasing we can
use~\eqref{eq:logregclosedform} and write
\begin{align*}
   \AR(\theta) &= \EE_X \max_{\delta \in \pertset{\infty}(\epsilon)} \indicator\{\sgn(\langle X,
    \thetastar\rangle) \langle X+\delta, \theta\rangle \leq 0\} \\
    &= \EE_X \indicator\{ \sgn(\langle X, \thetastar\rangle) \langle X, \theta\rangle - \eps \|\Piper \theta\|_1 \leq 0\}.
\end{align*}
Let $\EPipar \define \frac{1}{\|\theta\|_2^2}
\theta \theta^\top$ be the projection onto the subspace spanned by
$\theta$ and $\EPiper \define \idmat_d - \EPipar$ the projection onto the
orthogonal complement. Since $X$ is a vector with \iid{}~standard
normal distributed entries, we equivalently have
\begin{equation}
  \AR(\theta) = \EE_{Z_1,Z_2} \indicator\{ Z_1
   \sgn\left(Z_1
\|\EPipar \thetastar\|_2+ Z_2 \| \EPiper
\thetastar\|_2\right) - \eps \frac{\| \Piper \theta
   \|_1}{\|\theta\|_2} \leq 0\},
\end{equation}
with $Z_1,Z_2$ two independent standard normal random
variables.  For brevity of notation, define $\nu = \eps \frac{\|
  \Piper \theta \|_1}{\|\theta\|_2} $ and
$\Izozt = \sgn\left(Z_1
\|\EPipar \thetastar\|_2+ Z_2 \| \EPiper
\thetastar\|_2\right) =: \sgn( \beta^\top Z)$ with $\beta^\top = (\|\EPipar \thetastar\|_2, \| \EPiper \thetastar\|_2)$.


Define the event $\theevent = \{\sgn\left(Z_1 \|\EPipar
\thetastar\|_2+ Z_2 \| \EPiper \thetastar\|_2\right) - \eps \frac{\|
  \Piper \theta \|_1}{\|\theta\|_2} \leq 0\}$.  Because $Z_2$ is symmetric,
the distribution
of $Z_1 \Izozt$ is symmetric, hence we can rewrite the risk
\begin{equation}
  \AR(\theta) = \underbrace{\prob( \Izozt \leq 0 \vert Z_1 \geq 0)}_{T_1} + \underbrace{\prob(Z_1 \leq \nu, \Izozt \geq 0 \vert Z_1 \geq 0)}_{T_2}
\end{equation}
and derive expression for $T_1, T_2$ separately.

\paragraph{Step 1: Proof for $T_1$}
Note that due to $\|\thetatrue\|_2=1$ we have $\|\beta\|_2=1$
and recall that ${T_1 = \prob( \beta^\top Z\leq 0 \vert Z_1 \geq 0)}$.
Using the fact that both $Z_1$ and
$Z_2$ are independent standard normal distributed random variables, a simple geometric argument then yields that $T_1 =
\frac{\alpha}{\pi}$ with $\alpha =
\arc\cos\left(\frac{\beta_1}{\sqrt{\beta_1^2 +\beta_2^2}}\right) =
\arc\cos(\beta_1)$. Noting that $\beta_1 = \|\EPipar \thetastar\|_2 = \frac{\langle \thetastar, \theta\rangle}{\|\theta\|_2}$ then yields $T_1 =  \frac{1}{\pi}\arccos\left(\frac{\langle\thetanull, \theta\rangle}{\|\theta\|_2}\right)$.

\paragraph{Step 2: Proof for $T_2$}
First, assume that $ \langle \thetastar, \theta\rangle \geq 0$. We
separate the event ${\Vsetpr = \{Z_1 \leq \nu, \Izozt \geq 0\}}$ into
  two events $\Vsetpr = \Vsetpr_1 \cup \Vsetpr_2$
\begin{equation*}
  {\Vsetpr_1 = \{ Z_1 \leq \nu, Z_2 \geq 0\}}
  ~~\mathrm{and}~~{\Vsetpr_2 = \{ Z_1 \leq \nu, \Izozt \geq 0, Z_2
    \leq 0\}}.
\end{equation*}
The conditional probability of the first event is directly given
\begin{equation*}
  \prob(\Vsetpr_1 |
  Z_1\geq 0) = \prob(Z_2 \geq 0) \prob(Z_1\leq \nu \vert Z_1\geq 0) =\frac{1}{2} \erf\left(\frac{\nu}{\sqrt{2}}\right).
\end{equation*}
Hence
it only remains to find an expression for $\prob(\Vsetpr_2 \vert Z_1
\geq 0)$. Letting $\mu$ denote the standard normal distribution, we can write
\begin{align*}
  \prob(Z_1 \leq \nu, Z_2 &\leq 0, \Izozt \geq 0\vert Z_1 \geq 0) = 2 \int_0^{\nu} \int_{0}^{\frac{\beta_1 x}{\beta_2}} d\mu(y)d\mu(x)
 = \int_0^\nu \frac{1}{2} \erf\left(\frac{\beta_1 x}{\beta_2}\right) d\mu(x).
\end{align*}
Together with Step 1, Equation~\eqref{eq:robriskclass} follows by noting that $\beta_1^2 + \beta_2^2
= 1$.  Finally, the proof for the case where $\langle \thetastar,
\theta\rangle \leq 0$ follows exactly from the same argument and the proof is complete.
\end{proof}

\subsection{Distribution shift robustness and consistent adversarial robustness}
\label{sec:dro}

In this section we introduce distribution shift robustness
and show the relation to consistent $\ell_p$-adversarial robustness
for certain types of distribution shifts.

When learned models are
deployed in the wild, test and train distribution might not be be the same. That
is, the test loss might be evaluated on samples from a slightly
different distribution than used to train the method. Shifts in the mean of the covariate
distribution is a standard intervention studied in the invariant causal
prediction literature~\cite{Buhlmann20, Chen20}.  For mean
shifts in the null space of the ground truth $\thetastar$
we define an alternative evaluation metric that we refer to as the
\emph{distributionally robust risk} defined as follows:

\begin{align*}
  \DR(\theta) &\define \max_{\probrobx \in \probneighbor{q}(\epsilon;\probx)}
  \EE_{X\sim \probrobx}  \losstest(\langle \theta, X \rangle, \langle \thetastar, X\rangle), \:\text{ with }\\
  \probneighbor{p}(\eps;\prob) &\define \{\probrob \in \probspace: \|\mu_{\prob}
  -\mu_{\probrob}\|_p\leq \epsilon \text{ and } \langle \mu_{\prob} -
\mu_{\probrob}, \thetatrue\rangle = 0\}, \nonumber
\end{align*}
where $\probneighbor{p}$ is the neighborhood of mean shifted probability
distributions.

A duality between distribution shift robustness and
adversarial robustness has been established in earlier work such
as~\cite{sinha18} for general convex, continuous losses
$\losstest$. For our setting, the following lemma holds.

\begin{lemma}
  \label{lem:dro}
  For any $\eps \geq 0$ and $\theta$, we have $\DR(\theta) \leq \AR(\theta)$.
\end{lemma}


\begin{proof}
  The proof follows directly from the definition and consistency of the perturbations $\pertset{p}(\eps)$ and orthogonality of the mean shifts for the neighborhood $\probneighbor{p}$. By defining the random variable $W = X - \mu_\prob $ for $X \sim \prob$ we have the distributional equivalence
  \begin{equation*}
    X' = \mu_\prob + \delta + W\stackrel{d}{=} x + \delta
  \end{equation*}
  for $X'\sim \probrob$ and $X\sim \prob$ with $\mu_\probrob - \mu_\prob = \delta$ and hence
  \begin{align*}
    \DR(\theta) = \max_{\probrob \in \probneighbor{p}(\eps)} \EE_{X\sim\probrob} \losstest(\langle \theta, X\rangle, \langle \thetatrue, X\rangle) &= \max_{\|\delta\|_p \leq \eps, \delta \perp \thetatrue} \EE_{X\sim\prob} \losstest(\langle \theta, X+\delta \rangle, \langle \thetatrue,X\rangle) \\
    &\leq \EE_{X\sim\prob}  \max_{\|\delta\|_p \leq \eps, \delta \perp \thetatrue} \losstest(\langle \theta, X+\delta \rangle, \langle \thetatrue, X\rangle)
    = \AR(\theta)
  \end{align*}
  where the first line follows from orthogonality of the mean-shift to $\thetatrue$.
\end{proof}

%% file: appendix/experiments.tex
\section{Experimental details}
\label{sec:appendix_experiments}
In this section we provide additional details on our experiments.
All our code including instructions and hyperparameters can be
found here: \url{https://github.com/michaelaerni/interpolation_robustness}.

\subsection{Neural networks on sanitized binary MNIST}
Figure~\ref{fig:teaser_plot_nn} shows that robust overfitting
in the overparameterized regime also occurs for single hidden layer neural
networks on an image dataset that we chose to be arguably devoid of noise.
We consider binary classification of MNIST classes 1 vs 3
and further reduce variance by removing ``difficult'' samples.
More precisely, we train networks of width $p \in \{10^1, 10^2, 10^3\}$
on the full MNIST training data
and discard all images that take at least one of the models
more than 100 epochs of training to fit.
While some recent work argues that such sanitation procedures
can effectively mitigate robust overfitting \cite{Dong21},
we still observe a significant gap between the best (early-stopped)
and final test robust accuracies in Figure~\ref{fig:teaser_plot_nn}.

We train all networks on a subset of $n=2 \times 10^3$ samples
using plain mini-batch stochastic gradient descent
with learning rate $\nu_p = \sqrt{0.1 / p}$
that we multiply by $0.1$ after 300 epochs.
This learning rate schedule minimizes the training loss efficiently;
we did not perform tuning using test or validation data.
For the robust test error, we approximate worst-case $\ell_\infty$-perturbations
using 10-step SGD attacks on each test sample.

\subsection{Linear and logistic regression}

If not mentioned otherwise,
we use noiseless \iid{}\ samples
from our synthetic data model
as described in Section~\ref{sec:data_model}
for our empirical simulations.
We calculate all risks in closed-form without noise and,
in the robust case, with consistent perturbations.
However, we approximate the integral for the robust 0-1 risk
in Theorem~\ref{theo:logreg} using a numerical integral solver
since we cannot obtain a solution analytically.

For linear regression, we always sample a training set of size $n=10^3$
and run zero-initialized gradient descent for $2 \times 10^3$ iterations.
The learning rate depends on the data dimension $d$ as $\nu_d = \sqrt{1 / d}$.
Since we observed the training to be initially unstable
for large overparameterization ratios $d/n$,
we linearly increase the learning rate from zero
during the first 250 gradient descent iterations.
For evaluation, the linear regression robust population risk always uses
consistent $\ell_2$-perturbations of radius $\epstest = 0.4$.
For the noisy case in Figure~\ref{fig:teaser_plot_linreg}
we set $\sigma^2 = 0.2$.

We fit all logistic regression models
except in Figure~\ref{fig:early_stopping}
by minimizing the (regularized) logistic loss from Equation~\eqref{eq:logridge}
using \emph{CVXPY} in combination with the \emph{Mosek} convex programming solver.
Whenever the max-$\ell_2$-margin solution is feasible for $\lambda \to 0$,
the problem in Equation~\eqref{eq:logridge} has many optimal solutions.
In that case, we directly optimize the constrained problem
from Equation~\eqref{eq:maxmarginAE} instead.
For Figure~\ref{fig:early_stopping},
we run zero-initialized gradient descent
on the unregularized loss ($\lambda = 0$) for $5 \times 10^5$ iterations.
We start with a small initial step size of $0.01$
that we double every $3 \times 10^4$ steps until iteration $3 \times 10^5$.
Next, we perform all simulations in Figure~\ref{fig:logreg_explanation}
on $n=10^3$ samples from our data model with $d=8 \times 10^3$.
Both training and evaluation use consistent $\ell_\infty$-perturbations
of radius $\eps=0.1$.
Lastly, for the noisy case in Figure~\ref{fig:teaser_plot_logreg},
we flip 2\% of all training sample labels.

\subsection{Theoretical predictions}

In order to obtain the asymptotic theoretical predictions for logistic
regression in Figure~\ref{fig:logreg_theory_combined} corresponding to
the empirical simulations with $n = 10^3$ and $\eps = 0.05$,
we obtain the solution of the optimization problems in
Theorem~\ref{thm:nonseperable},\ref{thm:seperable}
with $\eps_0 = 0.05 \sqrt{10^3 \gamma}$
by solving the system of equations $\nabla C = 0$ (with $C$ the optimization objective form Theorem~\ref{thm:nonseperable},\ref{thm:seperable})
using \emph{MATLAB}'s optimization toolbox
where we approximate expectations via numerical integration. We note that the optimization problem is numerically challenging to solve, in particular for small values of $\gamma$.

%% file: appendix/linreg_exp.tex
\section{Linear regression -- additional insights}
\ifarxiv
\label{sec:linreg_inconsistent}

In this section we discuss how inconsistent adversarial training prevents interpolation for linear regression.
\fi
As shown in \cite{Javanmard20a} and using the same arguments as in Section~\ref{sec:linrobrisk},
the robust square loss under inconsistent $\ell_2$-perturbations
can be reformulated as
\begin{align*}
    \empriskrobunreg(\theta) &=
        \frac{1}{n} \sum_{i=1}^n (\vert\y_i - \langle x_i, \theta \rangle \vert + \epstrain \|\theta\|_2)^2
    \nonumber \\
    &=
    \frac{1}{n} \sum_{i=1}^n (y_i - \langle x_i, \theta \rangle )^2
    + \epstrain^2 \|\theta\|_2^2
    + \frac{2 \epstrain}{n}
    \|\theta\|_2 \sum_{i=1}^n \vert\y_i - \langle x_i, \theta \rangle \vert .
    \label{eq:linreg_robust_estimator}
\end{align*}
Thus, we can see that adversarial training with inconsistent perturbations prevents interpolation even when $d>n$,
that is, $\empriskrobunreg(\theta)=0$ is unattainable for any $\eps > 0$.
Nevertheless, we note that optimizing the reformulated robust square loss
is equivalent to
$\ell_2$-regularized linear regression with $\lambda = \epstrain^2$
and an additional term involving both the weight norm and absolute
prediction residuals.
We can observe this effect in Figure~4 of~\cite{Javanmard20a}
where larger $\epstrain$ yield similar effects to larger ridge penalties $\lambda$.

%% file: appendix/logreg_exp.tex
\section{Logistic regression -- additional insights}
In this section we further discuss robust logistic regression studied in Section~\ref{sec:logreg}.
\suppmat{}~\ref{sec:inconsistent} presents further experiments to
contrast consistent and inconsistent perturbations for adversarial training.
Furthermore, for completeness, we investigate standard training
(that is, $\epstrain = 0$)
in Section~\ref{sec:logreg_st_vs_at} and note that it yields significantly
worse standard and robust prediction performance.

\subsection{Inconsistent adversarial training}
\label{sec:inconsistent}

As observed in Section~\ref{sec:non_separable_data},
label noise can prevent interpolation
and hence improve the robust risk of an unregularized estimator
with $\lambda \to 0$.
We now show similar empirical effects of inconsistent training perturbations
with large enough radius $\epstrain$.

\begin{figure}[h]
    \centering
    \begin{subfigure}[l]{0.48\textwidth}
        \centering
        \begin{subfigure}[t]{\textwidth}
            \centering
            \includegraphics[width=2.6in]{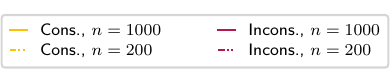}
        \end{subfigure}
        \begin{subfigure}[b]{\textwidth}
            \centering
            \includegraphics[width=2.6in]{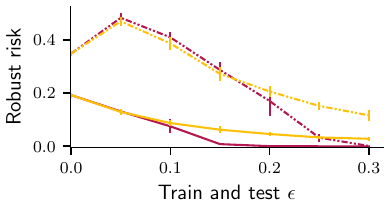}
        \end{subfigure}
        \caption{Varying $\eps$ for different $n$}
        \label{fig:logreg_inconsistent_vs_consistent_eps_increase}
    \end{subfigure}
    \begin{subfigure}[r]{0.48\textwidth}
        \centering
        \begin{subfigure}[t]{\textwidth}
            \centering
            \includegraphics[width=2.6in]{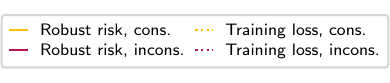}
        \end{subfigure}
        \begin{subfigure}[b]{\textwidth}
            \centering
            \includegraphics[width=2.6in]{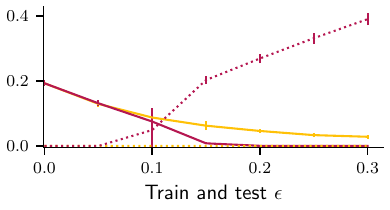}
        \end{subfigure}
        \caption{Robust risk vs.~training loss}
        \label{fig:logreg_inconsistent_vs_consistent_eps_increase_loss}
    \end{subfigure}
    \caption{
        Comparison of logistic regression adversarial training
        with consistent vs.\ inconsistent $\ell_\infty$-perturbations.
        (a) Robust risks of the unregularized estimators $(\lambda \to 0)$
        for $n=200, 1000$ as $\eps$ increases.
        While for small $\eps$, consistent and inconsistent perturbations
        yield similar robust risks,
        inconsistent perturbations with large $\eps$
        outperform consistent perturbations in terms of robustness.
        We provide an explanation in (b) where we focus on $n=1000$.
        In contrast to training with consistent perturbations,
        inconsistent perturbations may prevent the training loss
        from vanishing as $\eps$ grows large enough.
        Hence, inconsistent training perturbations
        can induce spurious regularization effects.
        We average all experiments over five
        independent dataset draws
        from our data model with fixed $d=500$
        and indicate one standard deviation via error bars.
    }
    \label{fig:logreg_inconsistent_vs_consistent_2}
\end{figure}

Concretely, we perform unregularized ($\lambda \to 0$)
adversarial training using consistent vs.\ inconsistent
$\ell_\infty$-perturbations for different $\eps$
on fixed $n=200,10^3$ and $d=500$.
Figure~\ref{fig:logreg_inconsistent_vs_consistent_eps_increase} displays
the robust risks of the resulting estimators.
For small $\eps$,
all risks behave very similarly,
further corroborating our observations in Figure~\ref{fig:logreg_inconsistent_vs_consistent}.
However, as the perturbation radius $\eps$ grows large,
inconsistent perturbations for unregularized adversarial training
yield estimators with better robust risk compared to consistent perturbations.

In order to understand this phenomenon,
we focus on $n=10^3$ and depict the robust risk
as well as the robust (unregularized) logistic training loss
in Figure~\ref{fig:logreg_inconsistent_vs_consistent_eps_increase_loss}.
We observe that, for large $\eps$,
inconsistent adversarial training fails to achieve a vanishing loss.
Hence, large enough inconsistent perturbations induce noise
which prevents interpolation.
This observation is similar to the observation made in Section~\ref{sec:non_separable_data}, where explicit label noise can have spurious regularization effects
and in turn, lead to a lower robust risk.

\subsection{Standard vs.\ adversarial training}
\label{sec:logreg_st_vs_at}

\begin{figure}
    \centering
    \begin{subfigure}[t]{\textwidth}
        \centering
        \includegraphics[width=3.9in]{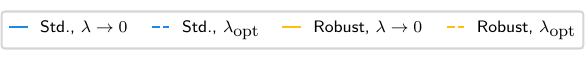}
    \end{subfigure}
    \begin{subfigure}[l]{0.48\textwidth}
        \centering
        \includegraphics[width=2.6in]{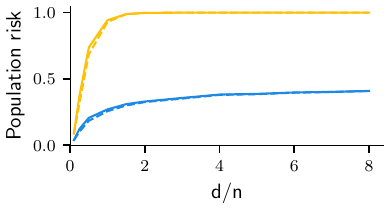}
        \caption{Standard training and $\ell_\infty$-perturbations}
        \label{fig:logreg_standard_training}
    \end{subfigure}
    \begin{subfigure}[r]{0.48\textwidth}
        \centering
        \includegraphics[width=2.6in]{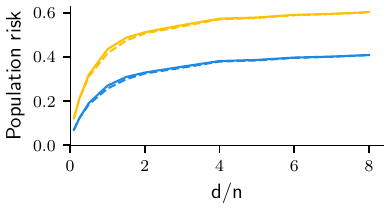}
        \caption{Adversarial training with $\ell_2$-perturbations}
        \label{fig:logreg_l2_perturbations}
    \end{subfigure}
    \caption{
        Additional logistic regression simulations
        using $n=10^3$ training samples from our data model
        for varying degrees of overparameterization $d/n$.
        (a) Standard training evaluated using
        consistent $\ell_\infty$-perturbations of radius $\epstest=0.1$.
        (b) Adversarial training
        using inconsistent $\ell_2$-perturbations
        of radius $\eps = 0.5$ for training
        and the corresponding consistent perturbation set for evaluation.
        In both settings,
        the robust risks are large and not even an optimally
        weighted ridge penalty helps to reduce them.
    }
    \label{fig:logreg_st_and_l2_perturbations}
\end{figure}

Throughout this paper, we focus on adversarial training for logistic regression.
For completeness,
we also provide simulation results for standard training ($\epstrain=0$)
in Figure~\ref{fig:logreg_standard_training}.
We again use a dataset of size $n=10^3$.
In contrast to adversarial training with $\epstrain > 0$,
we do not observe overfitting
for neither the standard nor robust risk.
However, for $d > n$, the robust risk exhibits its maximum possible value
and hence fails to provide any insights.
We note that our observations are consistent with~\cite{Salehi19}.

\subsection{Adversarial training with \texorpdfstring{$\ell_2$}{l2}-perturbations}
\label{sec:logreg_l2_perturbations}

As mentioned in Section~\ref{sec:logreg},
we focus on $\ell_\infty$-perturbations
in the context of logistic regression but for completeness also discuss $\ell_2$-perturbations.
Following the same argument as in Lemma~\ref{lm:logregrisk},
it is trivial to see that $\ell_2$-perturbations punish
the $\ell_2$-norm of the estimator.
Intuitively, we therefore expect that adversarial training
with $\ell_2$-perturbations results in an estimator $\thetalambda$
that is close to a rescaled version of the estimator obtained
if training without adversarial perturbations.
Since both the robust and standard risk are
independent of the estimator scale,
we hence do not expect any benefits
from explicit $\ell_2$-regularization, i.e., no robust overfitting.
Indeed, our simulation results
in Figure~\ref{fig:logreg_l2_perturbations}
show almost no regularization benefits
for neither the standard nor robust risk.

%% file: appendix/linreg_theory.tex
\section{Proof of Theorem~\ref{thm:main_thm_lr}}
\label{sec:linregtheory}
\label{sec:prooflinregthm}

In this section, we provide a proof of Theorem \ref{thm:main_thm_lr},
which characterizes the asymptotic risk of the linear regression
estimator $\thetalambda$ defined in Equation~\eqref{eq:linregridge}.

We first introduce some notation and give the standard closed form solution for
the ridge regression estimate $\thetalambda$.  Denoting the input data
matrix by $\Xs \in \RR^{d \times n}$, the observation vector
$y\in\R^n$ reads $y = \Xs^\top \thetastar + \xi$ where $\xi\sim
\Normal(0,I)$ is the noise vector. The noise vector contains \iid{}
zero-mean $\sigma^2$-variance Gaussian noise as entries.  Defining the
empirical covariance matrix as $\Sx = \frac{1}{n} \Xs^\top \Xs$ yields
the ridge estimate
\begin{equation}
\begin{split}
\thetalambda &= \frac{1}{n} (\tlambda \idmat_d + \Sx)^{-1} \Xs^\top y 
 \\
 &= (\tlambda \idmat_d + \Sx)^{-1} \Sx \thetastar + \frac{1}{n} (\tlambda \idmat_d + \Sx)^{-1} \Xs^\top \xi.
 \label{eq:ridgeclosedform}
 \end{split}
\end{equation}
For $\lambda \to 0$, we obtain the min-norm interpolator
\begin{equation*}
  \thetaminnorm = \lim_{\lambda \to 0} \thetalambda 
  = \Sx^{\dagger} \Xs^\top y,
\end{equation*}
where $\Sx^{\dagger}$ denotes the Moore-Penrose pseudo inverse.

We now compute the adversarial risk of this estimator. By Equation
\eqref{eq:ARreg}, the adversarial risk depends on the estimator only
via the two terms $\|\thetalambda - \thetastar\|_2$ and $\|\Piper
\thetalambda \|_2$.  To characterize the asymptotic risk, we hence
separately derive asymptotic expressions for each of both terms.
The following convergence results hold almost surely with respect
to the draws of the train dataset, with input features $\Xs$ and
observations $y$, as $n,d \to \infty$.

\paragraph{Step 1: Characterizing $\| \thetalambda  - \thetastar\|^2_2$.} 
Here, we show that 
\begin{align}
\label{eq:normdiffest}
\| \thetalambda  - \thetastar\|^2_2 \to \Riskinfty = \Biasinfty + \Varinfty,
\end{align}
where ${\Biasinfty = \lambda^2 m'(-\lambda)}$ and ${\Varinfty = \sigma^2
  \gamma (m(-\lambda) - \lambda m'(-\lambda))}$ are the asymptotic
bias and variance. \citet{Hastie19} considers a similar setup and Theorem 5 of \cite{Hastie19} show
 that ${\EE_{\xi} \|\thetalambda -
  \thetastar\|^2_2 \to \Biasinfty + \Varinfty}$ and the expectation is
taken over the observation noise $\xi$ in the train dataset. In this
paper, we define the population risks
without the expectation over the
noise. Hence, in a first step, the goal is to extend Theorem 5
\cite{Hastie19} for the standard risk $\SR(\thetalambda) = \|\thetalambda - \thetastar\|^2_2$ such that \eqref{eq:normdiffest} holds almost surely
over the draws of the training data.


Using Equation~\eqref{eq:ridgeclosedform} we can rewrite
\begin{align*}
 \| \thetalambda  - \thetastar\|^2_2 &= \| \left(\idmat_d - (\tlambda \idmat_d + \Sx)^{-1} \Sx\right) \thetastar + \frac{1}{n} (\tlambda \idmat_d + \Sx)^{-1} \Xs^\top \xi\|_2^2 \\
  &= \underbrace{\| \left(\idmat_d - (\tlambda \idmat_d + \Sx)^{-1} \Sx\right) \thetastar \|_2^2}_{T_1} + \underbrace{\langle \frac{\xi}{\sqrt{n}}, (\tlambda \idmat_d + \Sx)^{-2}\Sx \frac{\xi}{\sqrt{n}}\rangle}_{T_2} \\
 &+   \underbrace{\left\langle \frac{\Xs^\top}{\sqrt{n}} (\tlambda \idmat_d + \Sx)^{-1}  \left(\idmat_d - (\tlambda \idmat_d + \Sx)^{-1} \Sx\right) \thetastar,    \frac{\xi}{\sqrt{n}} \right\rangle}_{T_3},
\end{align*}
where we used for the second equality that ${\langle \frac{\xi}{\sqrt{n}}, \frac{X}{\sqrt{n}}  (\tlambda \idmat_d + \Sx)^{-2}\frac{X^\top}{\sqrt{n}} \frac{\xi}{\sqrt{n}}\rangle = \langle \frac{\xi}{\sqrt{n}}, (\tlambda \idmat_d + \Sx)^{-2}\Sx \frac{\xi}{\sqrt{n}}\rangle}$.

The first term $T_1 \to \Biasinfty$ follows directly via Theorem 5
\cite{Hastie19}.  We next show that $T_2 \to \Varinfty$ and $T_3\to 0$
almost surely, which establishes Equation~\ref{eq:normdiffest}.


\paragraph{Proof that $T_2\to \Varinfty$:}
While Theorem 5 \cite{Hastie19} also shows that ${\EE_{\xi}
  \tr\left(\frac{1}{n} \xi \xi^\top \Sx (\tlambda \idmat_d +
  \Sx)^{-2}\right) \to \Varinfty}$, we require the convergence almost
surely over a single draw of $\xi$. In fact, this directly follows from the
same argument as used for the proof of Theorem 5 \cite{Hastie19} and
the fact that $\| \frac{\xi}{\sqrt{n}} \|_2^2 \to \sigma^2$. Hence
${\langle \frac{\xi}{\sqrt{n}}, (\tlambda \idmat_d + \Sx)^{-2}\Sx
  \frac{\xi}{\sqrt{n}}\rangle \to \Varinfty}$ almost surely over the
draws of $\xi$.



\paragraph{Proof that $T_3\to 0$:} This follows straight forwardly from sub-Gaussian concentration inequalities and from the fact that
\begin{equation*}
 \left\|  \frac{\Xs}{\sqrt{n}} (\tlambda \idmat_d + \Sx)^{-1}  \left(\idmat_d - (\tlambda \idmat_d + \Sx)^{-1} \Sx\right) \thetastar \right\|_2 = O(1),
\end{equation*}
which is a direct consequence of the Bai-Yin theorem \cite{Bai93},
stating that for sufficiently large $n$, the non zero eigenvalues of
$\Sx$ can be almost surely bounded by $(1+\sqrt{\gamma})^2 \geq
\lambda_{\max}(\Sx) \geq \lambda_{\min}(\Sx) \geq
(1-\sqrt{\gamma})^2$. Hence we can conclude the first part of the
proof.

\paragraph{Step 2: Characterizing $\|\Piper \thetalambda\|_2$. } 
Here, we show that
\begin{align}
  \label{eq:vaduz}
\|\Piper \thetalambda\|_2^2 \to  \Riskinfty -\lambda^2 (m(-\lambda))^2 =:\Pperpinfty.
\end{align}
We assume without loss of generality that $\|\thetastar \|_2 =1$ and hence $\Piper = \idmat_d - \thetastar (\thetastar)^\top$. It follows that
\begin{align*}
 \|\Piper \thetalambda\|_2^2 
 &= \|\thetalambda\|_2^2 - \left(\langle \thetalambda, \thetastar \rangle\right)^2 \\
&= \|\thetastar - \thetalambda - \thetastar\|_2^2 - \left(1 - \langle \thetastar - \thetalambda, \thetastar \rangle\right)^2 \\
 &= \|\thetastar - \thetalambda\|_2^2 - 2 \langle \thetastar - \thetalambda, \thetastar \rangle +1  - \left(1 - \langle \thetastar - \thetalambda, \thetastar \rangle\right)^2 \\
 &= \|\thetastar - \thetalambda\|_2^2 - \left(\langle \thetastar - \thetalambda, \thetastar \rangle\right)^2.
\end{align*}
The convergence of the first term is already known form step 1. Hence,
it is only left to find an asymptotic expression for $\langle
\thetastar - \thetalambda, \thetastar \rangle$. Inserting the closed
form expression from Equation \eqref{eq:ridgeclosedform}, we obtain:
\begin{align}
  \label{eq:zurich}
 \langle \thetastar - \thetalambda, \thetastar \rangle = \langle \idmat_d - \left(\tlambda \idmat_d + \Sx)^{-1} \Sx\right) \thetastar, \thetastar \rangle - \langle  (\tlambda \idmat_d + \Sx)^{-1} \frac{\Xs^\top \xi}{n}, \thetastar \rangle.
\end{align}
Note that $\langle (\tlambda \idmat_d + \Sx)^{-1} \frac{\Xs \xi}{n},
\thetastar \rangle $ vanishes almost surely over the draws of $\xi$
using the same reasoning as in the first step. Hence, we only need to
find an expression for the first term on the RHS of
Equation~\eqref{eq:zurich}. Note that we can use Woodbury's matrix
identity to write:
\begin{align*}
 \langle \idmat_d - \left(\tlambda \idmat_d + \Sx)^{-1} \Sx\right) \thetastar, \thetastar \rangle =  \tlambda \langle (\tlambda \idmat_d + \Sx)^{-1} \thetastar, \thetastar \rangle. 
\end{align*}
However, the expression on the RHS appears exactly in the proof of Theorem 1 \cite{Hastie19} (Equation 116), which shows that
$ \tlambda \langle (\tlambda \idmat_d + \Sx)^{-1} \thetastar, \thetastar \rangle  \to \lambda m(-\lambda)$ with $m(z)$ as in Theorem \ref{thm:main_thm_lr}.  Hence the proof of almost sure convergence~\eqref{eq:vaduz} of $\|\Piper \thetalambda\|_2$ is complete.

Substituting Equations~\eqref{eq:normdiffest} and \eqref{eq:vaduz} into
robust risk \eqref{eq:ARreg} expression yields:
\begin{align*}
 \AR(\thetalambda) \: \overset{\text{a.s.}}{\longrightarrow} \:\Riskinfty +  \epstest^2  \Pperpinfty + \sqrt{\frac{8 \epstest^2}{\pi} \Pperpinfty  \Riskinfty} = \limitrisk.
\end{align*}
Finally, we note that $\lim_{\lambda \to 0} \limitrisk$ exists and is finite for any $\gamma \neq 1$ since:  $\lim_{\lambda \to 0} m(-z) = \frac{1
}{1-\gamma}$ for $\gamma <1$ and $\lim_{\lambda \to 0} m(-z) = \frac{1}{\gamma(\gamma-1)}$ for $\gamma >1$,  $\lim_{z\to 0} zm'(-z) = 0$, $\lim_{z \to
0} z^2m'(-z) = 0$ for $\gamma <1$ and $\lim_{z \to
0} z^2m'(-z)= 1 - \frac{1}{\gamma}$
for $\gamma >1$ (see also Corollary 5 in~\cite{Hastie19}). Hence, we can conclude from the continuity of the risk that $\AR(\thetahat_0)  \: \overset{\text{a.s.}}{\longrightarrow} \:\lim_{\lambda \to 0} \limitrisk$ and therefore, the proof is complete.


%% file: appendix/logreg_theory.tex
\section{Details on Theorem~\ref{theo:logreg}}
\label{sec:logregtheory}


In this section we give a formal statement for Theorem~\ref{theo:logreg}. The
results are based on the Convex Gaussian Minimax Theorem
(CGMT)~\cite{Gordon88,Thrampoulidis15}.
We first prove the case when training
with consistent perturbations~\eqref{eq:consistentpert} and noiseless observations.
Then, we show how the theorems extend to the case when training
with inconsistent perturbations~\eqref{eq:inconsistentpert} and
training label noise.

The results presented in this section
have similarities with the ones in~\cite{Javanmard20b}.
However, we study a discriminative data model
with features drawn from a single Gaussian and a
$1$-sparse ground truth.
In contrast, the authors of~\cite{Javanmard20b} study a generative data model
with features drawn from two Gaussians.
Furthermore, several papers study logistic regression for isotropic Gaussian
features in high dimensions~\cite{Salehi19,Sur19}, but focus their analysis
on the standard risk and do not consider adversarial robustness.

An immediate consequence of the proof of Lemma~\ref{lm:logregrisk} is
that the adversarial loss from Equation~\eqref{eq:logridge}
with respect to consistent $\ell_{\infty}$-attacks~\eqref{eq:consistentpert}
for the $1$-sparse ground truth has the closed-form equivalent
\begin{equation}
    \label{eq:logregrisk}
    \empriskrob(\theta) = \frac{1}{n} \sum_{i=1}^n \losstrain(y_i \langle\theta,\x_i\rangle - \eps\|\Piper \thetahat\|_1) + \lambda \|\theta\|_2^2,
\end{equation}
where $\Piper$ is the projection matrix to the orthogonal subspace of $\thetastar$.

Let $\Moreau_f(\x,t) = \min_{y} \frac{1}{2t}(x-y)^2 + f(y)$ be the Moreau envelope
and let $\Zpar,\Zperp$ be two independent standard normal random variables.
We can now state Theorem~\ref{thm:nonseperable} that describes the asymptotic risk
of $\thetaepslambda$, for $\lambda > 0$ and for
the asymptotic regime where $d,n \to \infty$.
The proof of the theorem can be found in \suppmat{}~\ref{sec:proofnonseperable}.

\begin{theorem}
    \label{thm:nonseperable}
    Assume that we have \iid{} random features $\x_i$
    drawn from an isotropic Gaussian,
    noiseless observations $\y_i = \sgn(\langle x_i, \thetastar\rangle)$,
    and ground truth $\thetastar = (1,0,\hdots,0)^\top$.
    Further, assume that $\lambda >0$ and $\eps = \eps_0/\sqrt{d}$, where $\eps_0$ is a numerical constant.
    Let $(\nuper^\star,\nupar^\star,  \orv^\star,\odelta^\star,\mu^\star, \tau^\star )$ be the unique solution of
    \begin{equation}
        \begin{split}
            \label{eq:nonseperable}
            & \min_{\substack{\nuper\geq 0, \tau\geq 0, \\ \nupar, \odelta \geq 0}}
            \max_{\substack{r\geq 0, \\ \mu\geq 0}} ~ \EE_{\Zpar,\Zperp}\left[ \Moreau_{\lossf}(\abs{\Zpar} \nupar + \Zperp \nuper - \eps_0 \odelta, \frac{\tau}{\orv})\right]  - \odelta \mu +\frac{\orv \tau}{2} + \lambda(\nuper^2 +\nupar^2) \\
            &-\nuper \sqrt{\left[ (\mu^2+\gamma\orv^2) -  (\mu^2+\gamma\orv^2) \erf(\mu/(\sqrt{\gamma}\orv\sqrt{2})) - \sqrt{\frac{2}{\pi}} \sqrt{\gamma}\orv\mu \exp(-\mu^2/(\gamma \orv^2 2)) \right]} .
        \end{split}
    \end{equation}
    Then, for $\lambda >0$, the estimator $\thetaepslambda$ from Equation~\eqref{eq:logridge}
    with the logistic loss and consistent $\ell_{\infty}$-perturbations satisfies
    asymptotically as $d,n\to\infty$ and $d/n \to \gamma$ that
    \begin{align}
        \label{eq:solnonsep}
        \frac{1}{\sqrt{d}} \|\Piper \thetaepslambda \|_1  \to \odelta^\star ~~\mathrm{and}~~
        \langle \thetaeps,  \thetanull \rangle  \to \nupar^\star ~~\mathrm{and}~~
        \|\thetaepslambda \|_2^2 \to \nupar^{\star 2} + \nuper^{\star 2}.
    \end{align}
    The convergences hold in probability.
\end{theorem}

For $\lambda >0$, the loss in Equation~\eqref{eq:logregrisk} has a unique minimizer.
In contrast, for $\lambda = 0$, the minimizer of Equation~\eqref{eq:logregrisk} is not unique.
In the latter case, we study the robust max-$\ell_2$-margin solution from Equation~\eqref{eq:maxmarginAE},
which corresponds to the limit when $\lambda \to 0$ (see Section~\ref{sec:estimator}).
Theorem \ref{thm:seperable} characterizes the asymptotic behavior of the corresponding solution,
with proof in \suppmat{}~\ref{sec:proofseperable}.

\begin{theorem}
    \label{thm:seperable}
    Assume that we have \iid{} random features $\x_i$
    drawn from an isotropic Gaussian,
    noiseless observations $\y_i = \sgn(\langle x_i, \thetastar\rangle)$,
    and $\thetastar = (1,0,\cdots,0)^\top$.
    Further, assume that $\lambda = 0$ and $\eps = \eps_0/\sqrt{d}$,
    where $\eps_0$ is a numerical constant.
    Let $(\nuper^\star,\nupar^\star, \orv^\star,\odelta^\star, \zeta^\star,\kappa^\star,\tau^\star)$ be the unique solution of
    \begin{equation}
        \begin{split}
        	\label{eq:thmsepeq}
     	\max_{\substack{r\geq 0, \\ \zeta\geq 0 }}
        \min_{\substack{\nuper\geq 0, \\ \nupar, \delta\geq 0}}
        ~&\max_{ \kappa\geq 0}
        \min_{\tau\geq 0}
        ~ \nupar^2 - \kappa \nuper -\odelta \zeta - \frac{\gamma \orv^2}{4(1+\frac{\kappa}{2\tau})} \\
        +& r \sqrt{\EE_{\Zpar, \Zperp}
        \left[ \max\left(0,1+\eps_0\delta - \abs{\Zpar} \nupar + \Zperp \nuper\right)^2 \right]} \\
        +& \frac{1}{2(1+\frac{\kappa}{2\tau})} \left(\frac{\gamma \orv^2+\zeta^2}{2}\erf\left(\frac{\zeta}{\sqrt{2} \sqrt{\gamma} \orv}\right) - \frac{\zeta^2}{2} + \frac{\sqrt{\gamma} \orv \zeta}{\sqrt{2 \pi}} \exp\left(- \frac{\zeta^2}{2 \gamma \orv^2} \right) \right) + \frac{\kappa\tau}{2}.
        \end{split}
    \end{equation}
    Then, the estimator $\thetaeps$ from Equation~\eqref{eq:maxmarginAE} with the logistic loss
    and consistent $\ell_{\infty}$-perturbations satisfies asymptotically as $d,n\to\infty$ and $d/n \to \gamma$ that
    \begin{align}
        \label{eq:solsep}
        \frac{1}{\sqrt{d}} \|\Piper \thetaeps \|_1  \to \odelta^\star ~~\mathrm{and}~~
        \langle \thetaeps,  \thetanull \rangle \to \nupar^\star ~~\mathrm{and}~~
        \|\thetaeps \|_2^2 \to \nupar^{\star 2} + \nuper^{\star 2}.
    \end{align}
    The convergences hold in probability.
\end{theorem}

\begin{remark}
Theorem~\ref{theo:logreg} follows from Theorems~\ref{thm:nonseperable}~and~\ref{thm:seperable}
when inserting the expression from Equations~\eqref{eq:solnonsep},\eqref{eq:solsep}
into the expression of the risk in Lemma~\ref{lm:logregrisk}.
\end{remark}

\paragraph{Inconsistent adversarial attacks}
\label{sec:inconsistentproof}
We now show that Theorems~\ref{thm:nonseperable},\ref{thm:seperable}
also hold when training with inconsistent attacks~\eqref{eq:inconsistentpert}.

For inconsistent adversarial attacks, we simply need to change
$\eps \|\Piper \theta\|_1$ to $\eps \| \theta\|_1 = \eps \|\Piper \theta\|_1 + \eps \| \Pipar \theta\|_1$
in the optimization objective in Equations~\eqref{eq:thmnonsepridge},\eqref{eq:f5proofmaxma}.
To show that these modifications do not change the asymptotic solution as $d,n \to \infty$,
note that ${\eps \| \Pipar \theta\|_1 = \frac{\eps_0}{\sqrt{n}} \| \Pipar \theta\|_1 \to 0}$
which follows from the fact that  $\| \Pipar \theta\|_1$ remains bounded as $d,n \to \infty$.

%
%


%

\paragraph{Label noise}
\label{sec:noiseprob}
While our results assume noiseless observations $y_i = \sgn(\langle x_i, \thetastar\rangle)$,
Theorem~\ref{thm:nonseperable},\ref{thm:seperable} can be extended to the case
where additional label noise is added to the observations.
That is, we observe $y_i = \sgn(\langle x_i, \thetastar\rangle) \xi_i$
with $\xi_i$ \iid{}, $\prob(\xi_i = 1) = 1- \sigma$
and $\prob(\xi_i = -1) = \sigma$, where $\sigma$ is the strength of the label noise.

Note that, as discussed in Section~\ref{sec:inconsistent},
the robust max-margin solution~\eqref{eq:maxmarginAE} might not exist for noisy observations.
In that case, the robust logistic regression estimate~\eqref{eq:logridge} has a unique solution for $\lambda =0$.
In fact, following the same argument as in~\cite{Javanmard20b}, asymptotically,
we can find a threshold $\gamma^\star$ such that for any $\gamma < \gamma^\star$,
the robust max-$\ell_2$-margin solution does not exist,
and for any $\gamma \geq \gamma^\star$, the robust max-$\ell_2$-margin solution exists.
The threshold can be found using the CGMT when following the same argument as in Theorem~6.1~of~\cite{Javanmard20b}.

Finally, we remark that,
when $\lambda >0$ or $\lambda = 0$ and $\gamma < \gamma^\star$,
we can extend Theorem~\ref{thm:nonseperable} by replacing $\abs{\Zpar}$ with $\xi \abs{\Zpar}$,
where $\xi$ is drawn from the same distribution as $\xi_i$ defined above.
Similarly, for $\lambda = 0$ and $\gamma \geq \gamma^\star$,
we can extend Theorem~\ref{thm:seperable} by replacing~$\abs{\Zpar}$ with~$\xi \abs{\Zpar}$.

\paragraph{Outline of the proof}

The proof of Theorems~\ref{thm:nonseperable},\ref{thm:seperable} heavily relies on the proofs
of Theorem~6.3~and~6.4 in~\cite{Javanmard20b}.
In particular, our proof essentially follows the same structure by first reducing the problem via an application of the Lagrange multiplier to an expression that suits the CGMT framework.
This allows us to instead study the auxiliary optimization problem as described in Equation~\eqref{eq:auxiliarycgmt},
which we then simplify to a scalar optimization problem using standard concentration inequalities
of Gaussian random variables.

The major difference to Theorems~6.3~and~6.4 in~\cite{Javanmard20b} is that we study a discriminative data model
with a sparse ground truth,
whereas Theorem~6.3~and~6.4 in~\cite{Javanmard20b} assume a generative data model
and, in particular, do not allow sparse ground truth vectors $\thetastar$.
This is due to the different attack sizes as we choose $\eps = \eps_0/\sqrt{d}$
while Theorems~6.3~and~6.4 in~\cite{Javanmard20b} assume a constant attack size $\eps = \eps_0$.

%
%
%
%
%
%
%
%
%

%
%
%
%

%
%
%
%
%
%
%

%

\subsection{Proof of Theorem \ref{thm:nonseperable}}

 \label{sec:proofnonseperable} Denote with $\Xs \in \RR^{n \times d}$ the input data matrix and with $y \in \RR^n$ the vector containing the observations.
Recall that the estimator $\thetahat$ is given by
\begin{align}
 \thetahat &= \arg\min_{\theta} \frac{1}{n} \sum_{i = 1}^{n} \lossf(\y_i \langle x_i, \theta \rangle - \eps  \|\Piper \theta\|_1) + \lambda \|\theta\|_2^2 \nonumber \\
 &= \arg\min_{\theta,v} \frac{1}{n} \sum_{i = 1}^{n} \lossf( v_i - \eps  \|\Piper \theta\|_1) + \lambda \|\theta\|_2^2 ~~\mathrm{such~that~} v  = D_y \Xs \theta, \label{eq:thmnonsepridge}
\end{align}
where $\lossf(x) = \log(1+\exp(-x))$ is the logistic loss, $\Xs \in \RR^{n\times
d}$ is the data matrix and $D_y$ the diagonal matrix with entries $(\Dy)_{i,i} =
\y_i$. We can then introduce the Lagrange multipliers $u \in \R^n$  to obtain
\begin{equation*}
 \min_{\theta,v} \max_{u} ~\frac{1}{n} \sum_{i=1}^{n} \lossf(v_i - \eps \|\Piper \theta\|_1) + \frac{1}{n} u^\top \Dy \Xs \theta - \frac{1}{n} u^\top v+ \lambda \|\theta\|_2^2.
\end{equation*}
%
%
%
Furthermore, we can separate $\Xs = \Xs \Piper + \Xs\Pipar$, which yields
\begin{align}
\label{eq:primalproblem}
 \min_{\theta,v} \max_{u} ~ \frac{1}{n} \sum_{i=1}^{n} \lossf(v_i - \eps \|\Piper \theta\|_1) + \frac{1}{n} u^\top \Dy \Xs \Pipar \theta
 +\frac{1}{n} u^\top \Dy \Xs \Piper \theta - \frac{1}{n} u^\top v + \lambda \|\theta\|_2^2.
\end{align}

\paragraph{Convex Gaussian Minimax Theorem}
We can now make use of the CGMT, which states that
\begin{equation}
    \label{eq:primalcgmt}
    \min_{\theta \in U_{\theta}} \max_{u \in U_{u}} u^\top \Xs \theta + \psi(u,\theta),
\end{equation}
with $\psi$ convex in $\theta$ and concave in $u$,
has asymptotically, when $d,n\to\infty$, $d/n \to \gamma$, pointwise the same solution as
\begin{equation}
    \label{eq:auxiliarycgmt}
    \min_{\theta \in U_{\theta}} \max_{u \in U_{u}} \|u\|_2 g^\top \theta + u^\top h \|\theta\|_2 + \psi(u,\theta),
\end{equation}
where $g \in \RR^d$ and $h\in \RR^n$ are random vectors with \iid{} standard normal entries,
and $U_{\theta}$ and $U_u$ are compact sets.
As is common in the literature, we call Equation~\eqref{eq:primalcgmt} the primal optimization problem
and Equation~\eqref{eq:auxiliarycgmt} the auxiliary optimization problem.
Several works have already used the CGMT to study high dimensional asymptotic logistic regression~\cite{Salehi19},
also when training with adversarial attacks~\cite{Javanmard20b}.
We omit the precise statement and refer the reader to~\cite{Thrampoulidis15}.
However, we note that we can apply the CGMT due to the following observations:
\begin{enumerate}
    \item The objective~\eqref{eq:primalproblem} is concave in $u$ and convex in $v,\theta$.
    \item We can restrict $u,v,\theta$ to compact sets without changing the solution.
    For $\theta$, we note that this is a consequence of $\lambda >0$,
    and for $u,v$, we note that the stationary condition requires $u_i = \lossf'(v_i - \eps \| \Piper \theta\|_1)$.
    \item $\Xs \Piper$ is independent of the observations $\y$ and of $\Xs \Pipar$.
\end{enumerate}
Therefore, as a consequence of the CGMT, we can show that the solution
of the primal optimization problem~\eqref{eq:primalproblem} asymptotically concentrates
around the same value as the solution of the following auxiliary optimization problem:

%
%
%

\begin{align*}
    \min_{\theta,v} \max_{u} ~ &\frac{1}{n} \sum_{i=1}^{n} \lossf(v_i - \eps \|\Piper \theta\|_1) + \frac{1}{n} u^\top \Dy \Xs \Pipar \theta +\frac{1}{n} \|u^\top \Dy\|_2 g^\top \Piper \theta \\ +&\frac{1}{n} u^\top \Dy h \| \Piper \theta \|_2  - \frac{1}{n} u^\top v + \lambda \|\theta\|_2^2,
\end{align*}
where $g\in\RR^d$ and $h\in \RR^n$ are vectors with \iid{} standard normal entries.

\paragraph{Scalarization of the optimization problem}
We now aim to simplify the optimization problem.
In a first step, we maximize over $u$. For this, define $\orv =  \|u\|_2/\sqrt{n}$, which allows us to equivalently write
\begin{equation*}
    \min_{\theta,v} \max_{\orv \geq 0} ~~ \frac{1}{n} \sum_{i=1}^{n} \lossf(v_i - \eps \|\Piper \theta\|_1) + \frac{r}{\sqrt{n}} \|\Dy \Xs \Pipar \theta  + \Dy h \| \Piper \theta \|_2  -  v \|_2 +\frac{1}{\sqrt{n}} \orv g^\top \Piper \theta + \lambda \|\theta\|_2^2,
\end{equation*}
where we have used the fact that $\|u^\top \Dy\|_2 = \|u\|_2$.
In order to proceed, we want to separate
$\Piper \theta$ from the loss $\lossf(v,\Piper \theta) := \frac{1}{n}
\sum_{i=1}^{n} \lossf(v_i - \eps \|\Piper \theta\|_1)$.
Denoting the conjugate of $\lossf$ by $\conjlossf$,
we can write $\lossf(v,\Piper \theta)$ in terms of its conjugate with respect to $\Piper \theta$:
\begin{align*}
    \lossf(v,\Piper \theta) &= \sup_{w} \frac{1}{\sqrt{d}}w^\top  \Piper \theta - \conjlossf(v,w) \\
    &= \sup_{w} \frac{1}{\sqrt{d}}w^\top  \Piper \theta -
    \sup_{\odelta \geq 0} \left(
    \frac{\sqrt{d}}{\sqrt{d}}\odelta \|w\|_{\infty} -
    \frac{1}{n} \sum_{i=1}^n \lossf(v_i - \sqrt{d}\eps
    \odelta) \right) \\
    &= \sup_{w} \inf_{\odelta \geq 0}
    \frac{1}{\sqrt{d}}w^\top  \Piper \theta - \odelta
    \|w\|_{\infty} + \frac{1}{n} \sum_{i=1}^n \lossf(v_i -
    \eps_0 \odelta),
\end{align*}
where, for the second identity, we use the derivation for the conjugate of $\lossf$
from Lemma~A.2 in the paper~\cite{Javanmard20b}.
Hence, we obtain:
\begin{align}
	\label{eq:delta}
	\max_{\substack{r\geq 0}} \min_{\substack{\theta,v }}\max_{w} \min_{\odelta\geq0}~~\frac{1}{n} &\sum_{i=1}^n \lossf(v_i - \eps_0 \odelta)  + \frac{r}{\sqrt{n}} \left\|\Dy \Xs \Pipar \theta + \Dy h \| \Piper \theta \|_2  -  v \right\|_2 + \lambda \|\theta\|_2^2 \\  +&  \frac{1}{\sqrt{d}}w^\top \Piper \theta - \odelta \|w\|_{\infty}  + \frac{1}{\sqrt{n}} \orv g^\top \Piper \theta.
	\end{align}
In particular, note that the problem is concave in $r, w$ and convex in $\theta,v,\delta$. Thus, we can 
interchange the order of maximization and minimization:
\begin{align}
		\max_{\substack{r\geq 0}} \min_{\substack{v}} \min_{\odelta\geq0}\max_{w}\min_{\theta}~~\frac{1}{n} &\sum_{i=1}^n \lossf(v_i - \eps_0 \odelta)  + \frac{r}{\sqrt{n}} \left\|\Dy \Xs \Pipar \theta + \Dy h \| \Piper \theta \|_2  -  v \right\|_2 + \lambda \|\theta\|_2^2 \\  +&  \frac{1}{\sqrt{d}}w^\top \Piper \theta - \odelta \|w\|_{\infty}  + \frac{1}{\sqrt{n}} \orv g^\top \Piper \theta   \nonumber.
\end{align}
Next, we simplify the optimization over $\theta$. Write $\Pipar \theta = \Pipar 1 \nupar$ with $\nupar \in \RR$ (here we use the fact that $\thetastar = (1,0,\cdots,0)$) and let $\nuper = \|\Piper \theta\|_2$. We can simplify:
\begin{align}
\label{eq:nuperpar}
 \max_{\substack{r\geq 0}} \min_{\substack{\nuper \geq 0,\\ \odelta\geq0, \\ \nupar, v}}\max_w   ~~&\frac{1}{n} \sum_{i=1}^n \lossf(v_i - \eps_0 \odelta)
  + \frac{r}{\sqrt{n}} \|\Dy \Xs \Pipar 1 \nupar  + \Dy h \nuper  -  v \|_2  + \lambda (\nupar^2 +\nuper^2)\\
  -&  \frac{1}{\sqrt{d}}\nuper \| \Piper(w - \sqrt{\gamma} \orv g)\|_2  -\odelta \|w\|_{\infty}  \nonumber
\end{align}
In order to obtain a low dimensional scalar optimization problem, we still need to scalarize the optimization over $w$ and $v$. For this, we replace the term $ \|\Dy \Xs \Pipar 1 \nupar  + \Dy h \nuper  -  v \|_2 $ with its square, which is achieved by using the following identity $\min_{\tau\geq0} \frac{x^2}{2\tau} + \frac{\tau}{2} = x$.
Hence,
\begin{align*}
\max_{\substack{r\geq 0}} \min_{\substack{\nuper \geq 0, \tau \geq 0\\ \odelta\geq0, \\ \nupar, v}}~~ &\frac{1}{n} \sum_{i=1}^n \lossf(v_i - \eps_0 \odelta)
+\frac{\orv}{2\tau n} \|\Dy \Xs \Pipar 1 \nupar   + \Dy h \nuper  -  v \|_2^2+ \frac{\tau \orv}{2}  + \lambda (\nupar^2 +\nuper^2)
  \\
  +& \max_w  \left[- \frac{1}{\sqrt{d}}\nuper \| \Piper(w - \sqrt{\gamma} \orv g)\|_2  -\odelta \|w\|_{\infty} \right].
\end{align*}
We can now separately solve the following two inner optimization problems:
\begin{align}
\label{eq:term1_opt}
 &\max_{w}~ - \nuper\frac{1}{d}\| \Piper(w - \sqrt{\gamma} \orv g)\|_2  -\odelta \|w\|_{\infty} \\
  \label{eq:term2_opt}
 &\min_{v}~ \frac{\orv}{2\tau n} \|\Dy \Xs \Pipar 1 \nupar +  \Dy h \nuper  -  v \|_2^2 +  \sum_{i=1}^n \lossf(v_i - \eps_0 \odelta)
\end{align}
\paragraph{Equation \eqref{eq:term1_opt}}
 Let $\SoftT_t(x) = \left\{\begin{matrix} 0 & \abs{x} \leq t \\ \sgn(x)(\abs{x}-t) &\mathrm{else}
\end{matrix}\right. $ be the soft threshold function. We have
\begin{align*}
 &\max_{w}~ -\nuper \frac{1}{\sqrt{d}}\| \Piper(w - \sqrt{\gamma} \orv g)\|_2  -\odelta \|w\|_{\infty} \\
 = - &\min_{w}~  \nuper\frac{1}{\sqrt{d}}\| \Piper(w - \sqrt{\gamma} \orv g)\|_2  +\odelta \|w\|_{\infty} \\
 \overset{\mu = \|w\|_{\infty}}{=} -&\min_{\mu \geq0} \nuper~ \sqrt{ \frac{1}{d} \sum_{i=2}^d \left(\SoftT_{\mu}\left(\sqrt{\gamma} \orv g_i\right)\right)^2} +\odelta \mu \\
 \overset{\mathrm{LLN~as~} d\to \infty}{\to} -&\min_{\mu \geq0} \nuper~ \sqrt{\EE_Z \left(\SoftT_{\mu}\left(\sqrt{\gamma} \orv Z \right)\right)^2} +\odelta \mu,
\end{align*}
where we used in the third line that the ground truth $\thetastar$ is $1$-sparse and in the last line that the expectation exists for $ Z \sim \Normal(0,1)$. Finally, we can further simplify
\begin{align*}
 \EE_Z \left(\SoftT_{\mu}\left(\sqrt{\gamma} \orv Z \right)\right)^2 &= \gamma r^2 \EE_Z \left(\SoftT_{\mu/(\sqrt{\gamma}r)}\left( Z \right)\right)^2 \\
 &= \gamma r^2 \EE_Z (Z- \mu/(\sqrt{\gamma} r))^2 - \EE_Z \indicator_{\abs{Z} \leq \mu/(\sqrt{\gamma} r)} (Z- \mu/(\sqrt{\gamma} r))^2 \\
 &=  (\mu^2+ \gamma r^2)\left(1 -  \erf(\mu/(\sqrt{2\gamma} r)) \right)- \sqrt{\gamma} r \mu \sqrt{\frac{2}{\pi}}  \exp(-\mu/(2 \gamma r^2 )).
\end{align*}
Hence, we can conclude the first term.

\paragraph{Equation \eqref{eq:term2_opt}}
For the second term we also aim to apply the law of large numbers. We have
\begin{align*}
  &\min_{v}~ \frac{\orv}{2\tau n} \|\Dy \Xs \Pipar 1 \nupar  + \Dy h \nuper  -  v \|_2^2 +  \frac{1}{n}\sum_{i=1}^n \lossf(v_i - \eps_0 \odelta) \\
  \overset{\tildev = v - \eps\odelta}{=} &\min_{\tildev}~  \frac{\orv}{2\tau n} \|\Dy \Xs \Pipar 1 \nupar +  \Dy h \nuper  -  \tildev - \eps_0 \odelta \|_2^2 +  \frac{1}{n} \sum_{i=1}^n \lossf(\tilde{v}_i) \\
  = &\min_{\tildev}~ \frac{1}{n} \sum_{i=1}^n \frac{\orv}{2\tau n} \left((\Dy \Xs \Pipar 1 \nupar)_i +  (\Dy h \nuper)_i  -  \tildev_i - \eps_0 \odelta \right)^2 +   \lossf(\tilde{v}_i) \\
  \overset{\mathrm{LLN}}{\to} &~~\EE_{\Zpar,\Zperp}\left[ \Moreau_{\lossf}(\abs{\Zpar}\nupar + \Zperp \nuper - \eps_0 \odelta, \frac{\tau}{\orv})\right],
\end{align*}
 where in the last line we used that $(\Dy\Xs \Pipar 1)_i = \y_i\x_i^\top\thetanull = \sgn(\x_i^\top\thetanull) \x_i^\top\thetanull$ has the same distribution as $\abs{\Zpar}$ with  $ \Zpar \sim \Normal(0,1)$. Further, to apply the law of large numbers, we need to show that the Moreau envelope exists. Similarly to Theorem 1 \cite{Salehi19}, this follows immediately when noting that ${\Moreau_{\lossf}(x,\mu) \leq \lossf(x) = \log(1+\exp(-x)) \leq \log(2) + \abs{x}}$.
Finally, we obtain the desired optimization problem in Equation \eqref{eq:nonseperable} when combining these results.

\paragraph{Convergence}
One can check that the optimization problems defined in
Equations~\eqref{eq:nonseperable},\eqref{eq:thmnonsepridge} are convex
in the variables that we minimize over,
and concave in the variables that we maximize over. Indeed,
Equation~\eqref{eq:thmnonsepridge} is immediate and
Equation~\eqref{eq:nonseperable} follows straightforwardly from the fact that
the problem in Equation~\eqref{eq:delta} is convex and concave as desired.
Therefore, also the problem in Equation ~\eqref{eq:nonseperable} satisfies the
convexity and concavity properties that we need. 
Hence, both problems in Equations~\eqref{eq:nonseperable},\eqref{eq:thmnonsepridge} have a unique solution.
Finally, note that the optimum $\delta^\star$ in Equation \eqref{eq:delta}
satisfies $\delta^\star = \frac{1}{\sqrt{d}} \| \Piper \theta\|$, and similarly
the optima $\nuper^\star$ and $\nupar^\star$ in Equation \eqref{eq:nuperpar}
satisfy $\nuper^\star = \| \Piper \theta\|_2$ and $\nupar^\star = \langle
\theta,  \thetanull \rangle$. We can therefore conclude the proof as the
solutions of the optimization problems~\eqref{eq:nonseperable}, \eqref{eq:thmnonsepridge}
concentrate asymptotically around the same optima as $d,n \to \infty$.

 \subsection{Proof of Theorem \ref{thm:seperable}}
 \label{sec:proofseperable}
Recall the robust max-margin solution from Equation~\eqref{eq:maxmarginAE}:
\begin{equation}
\label{eq:f5proofmaxma}
   \min_{\theta,\delta} \|\theta\|_2^2 ~~\mathrm{such~that}~ \langle \theta, \x_i
 \rangle - \delta \geq 1 ~\mathrm{for~all}~ i ~\mathrm{and}~ \eps \|\Piper \theta\|_1 = \delta
  \end{equation}
Like in the previous section, after introducing the Lagrange multipliers $\zeta$ and $u$ we can equivalently write
  \begin{equation}
  \label{eq:delta2}
   \min_{\theta, \delta}\max_{\substack{u: u_i\geq 0, \\ \zeta\geq 0}}
   \|\theta\|_2^2 + \frac{1}{n}u^\top \left( 1 + 1 \eps_0 \delta - D_y \Xs \theta\right) + \zeta\left(\frac{\| \Piper \theta|_1}{\sqrt{d}} -\delta \right),
   \end{equation}
   and again separating $\Xs=\Xs \Piper + \Xs \Pipar$, we get
   \begin{equation}
   \label{eq:primalsep}
      \min_{\theta, \delta}\max_{\substack{u: u_i\geq 0, \\ \zeta\geq 0}}
   \|\theta\|_2^2 + \frac{1}{n}u^\top \left( 1 + 1 \eps_0 \delta - D_y \Xs \Pipar \theta - D_y \Xs \Piper \theta \right) + \zeta \left(\frac{\|\Piper \theta\|_1}{\sqrt{d}} -\delta  \right) .
     \end{equation}
\paragraph{Convex Gaussian Minimax Theorem}
     Since the adversarial attacks are consistent and the observations are
     noiseless, we know the solution in Equation \eqref{eq:f5proofmaxma} exists
     for all $d,n$. Yet, in order to apply the CGMT, we have to show that we can
     restrict $u$ and $\theta$ to compact sets. This follows from a simple trick
     as explained in Section~D.3.1 in~\cite{Javanmard20b}. Hence, the primal
     optimization problem from Equation \eqref{eq:primalsep} can be
     asymptotically replaced with the following auxiliary optimization problem,
where, as before, $g \in \RR^d$ and $h\in \RR^n$ are random vectors with \iid{}
standard normal entries:
     \begin{equation}
  \begin{split}
  \label{eq:auxiliary2}
      \min_{\theta, \delta}
      \max_{\substack{u: u_i\geq 0, \\ \zeta\geq 0}}
      & \|\theta\|_2^2 + \frac{1}{n}u^\top
      \left( 1 + 1 \eps_0 \delta - D_y \Xs \Pipar \theta + D_y h \|\Piper \theta\|_2 \right) \\
      +& \frac{1}{n}\|u\|_2  g^\top \Piper \theta + \zeta\left(\frac{\|\Piper \theta\|_1}{\sqrt{d}} -\delta\right) .
   \end{split}
     \end{equation}
     \paragraph{Scalarization of the optimization problem}
     The goal is again to scalarize the optimization problem. As a first step, we can solve the optimization over $u$ when defining $r = \frac{\|u\|_2}{\sqrt{n}}$:
    \begin{equation*}
    \begin{split}
        \min_{\theta, \delta\geq 0}
        \max_{\substack{r\geq 0, \\ \zeta\geq 0}}
        & \|\theta\|_2^2 + \frac{r}{\sqrt{n}}
        \| \max\left(0, 1 + 1 \eps_0 \delta - D_y \Xs \Pipar \theta + D_y h \|\Piper \theta\|_2 \right)\|_2 \\
        +& \frac{r\sqrt{\gamma}}{\sqrt{d}} g^\top \Piper \theta + \zeta\left( \frac{1}{\sqrt{d}}\|\Piper \theta\|_1 -\delta\right),
     \end{split}
     \end{equation*}
where $\max$ applies element-wise over the vector.
We can now swap maximization and minimization
since the objective is convex in $\theta,\delta$ and concave in $r$:
    \begin{equation*}
	\begin{split}
		\max_{\substack{r\geq 0, \\ \zeta\geq 0}}\min_{\theta, \delta\geq 0}~
		&\|\theta\|_2^2 + \frac{r}{\sqrt{n}}\| \max\left(0, 1 + 1 \eps_0 \delta - D_y \Xs \Pipar \theta + D_y h \|\Piper \theta\|_2 \right)\|_2 \\
		+ &\frac{r\sqrt{\gamma}}{\sqrt{d}} g^\top \Piper \theta + \zeta\left( \frac{1}{\sqrt{d}}\|\Piper \theta\|_1 -\delta\right) .
	\end{split}
\end{equation*}
We now want to separate
$\|\Piper \theta\|_2$ from the term in $\max$.
This is achieved by introducing the variable $\nuper \geq 0$
and the Lagrange multiplier $\kappa$.
Further, we set $\nupar = \langle \thetastar, \Pipar \theta\rangle$
(recall that $\thetastar = (1,0,\cdots,0)$),
which allows us to equivalently write
\begin{equation}
\begin{split}
    \max_{\substack{r\geq 0, \\ \zeta\geq 0 }}
    \min_{\substack{\nuper\geq 0, \\ \nupar, \delta\geq 0, \Piper \theta}}
    ~\max_{ \kappa\geq 0}~
    & \nupar^2 + \|\Piper \theta\|_2^2 +\kappa(\|\Piper \theta\|_2 - \nuper) \\
    +& \frac{r}{\sqrt{n}}\| \max\left(0, 1 + 1 \eps_0 \delta - D_y \Xs \Pipar \theta^* \nupar + D_y h \nuper \right)\|_2 \\
    +& \frac{r\sqrt{\gamma}}{\sqrt{d}} g^\top \Piper \theta +  \zeta\left(\frac{1}{\sqrt{d}}\|\Piper \theta\|_1 -\delta\right) .
\end{split}
\end{equation}

Next, we use again the trick $\min_{\tau\geq0} \frac{x^2}{2\tau} + \frac{\tau}{2} = x$, which yields
\begin{equation}
\begin{split}
    \label{eq:nuperpar2}
    \max_{\substack{r\geq 0, \\ \zeta\geq 0 }}
    \min_{\substack{\nuper\geq 0, \\ \nupar, \delta\geq 0, \Piper \theta}}
    ~\max_{ \kappa\geq 0}
    \min_{\tau\geq 0}~
    & \nupar^2  + \|\Piper \theta\|_2^2 -\kappa \nuper + \frac{\kappa}{2\tau}\|\Piper \theta\|_2^2 \\
    &+ \frac{\kappa\tau}{2}+ \frac{r}{\sqrt{n}}\| \max\left(0, 1 + 1 \eps_0 \delta - D_y \Xs \Pipar \theta^* \nupar + D_y h \nuper \right)\|_2 \\
    &+ \frac{r\sqrt{\gamma}}{\sqrt{d}} g^\top \Piper \theta +  \zeta\left(\frac{1}{\sqrt{d}}\|\Piper \theta\|_1 -\delta\right) .
\end{split}
\end{equation}
In a next step, note that due to high dimensional concentration, we have that
     \begin{align*}
 &\frac{r}{\sqrt{n}}\| \max\left(0, 1 + 1 \eps_0 \delta - D_y \Xs \Pipar \theta^* \nupar + D_y h \nuper \right)\|_2 \\ \overset{\mathrm{LLN}}{\to}~ &r \sqrt{\EE_{\Zpar,\Zperp} \left[\max\left(0,1+\eps_0\delta - \abs{\Zpar} \nupar + \Zperp \nuper\right)^2\right]} =: \sqrt{T},
\end{align*}
where $\Zpar,\Zperp$ are standard Gaussian distributed random variable. 
Next, by completion of the squares we get
    \begin{equation*}
  \begin{split}
    \max_{\substack{r\geq 0, \\ \zeta\geq 0 }}  \min_{\substack{\nuper\geq 0, \\ \nupar, \delta\geq 0, \Piper \theta}}~\max_{ \kappa\geq 0}\min_{\tau\geq 0}~
   &\nupar^2  + (1+\frac{\kappa}{2\tau})\|\Piper \theta + \frac{r \sqrt{\gamma}}{\sqrt{d} 2(1+\frac{\kappa}{2\tau})} g\|_2^2 - \frac{r^2 \gamma}{ 4 (1+\frac{\kappa}{2\tau})} \|g/\sqrt{d}\|_2^2
   \\
   & -\kappa\nuper + \sqrt{T} +\zeta  \left(\frac{1}{\sqrt{d}}\|\Piper \theta\|_1 -\delta\right) + \frac{\kappa\tau}{2}
   \end{split}
     \end{equation*}
with $\|g/\sqrt{d}\|_2^2 \to 1$.  Next, note that we can again swap minimization and maximization due to the convexity and concavity, respectively, of the optimization. Hence, we can rewrite
    \begin{equation*}
	\begin{split}
		\max_{\substack{r\geq 0, \\ \zeta\geq 0 }}  \min_{\substack{\nuper\geq 0, \\ \nupar, \delta\geq 0}}~\max_{ \kappa\geq 0}\min_{\tau\geq 0, \Piper \theta}~
		&\nupar^2  + (1+\frac{\kappa}{2\tau})\|\Piper \theta + \frac{r \sqrt{\gamma}}{\sqrt{d} 2(1+\frac{\kappa}{2\tau})} g\|_2^2 - \frac{r^2 \gamma}{ 4 (1+\frac{\kappa}{2\tau})} \|g/\sqrt{d}\|_2^2
		\\
		& -\kappa\nuper + \sqrt{T} +\zeta  \left(\frac{1}{\sqrt{d}}\|\Piper \theta\|_1 -\delta\right) + \frac{\kappa\tau}{2} .
	\end{split}
\end{equation*}
Finally, to obtain the desired optimization problem, we only need to solve  the inner optimization over $\Piper \theta$.  For this, we can write:
\begin{equation*}
\begin{split}
 \min_{\Piper \theta}~&(1+\frac{\kappa}{2\tau})\|\Piper \theta + \frac{r \sqrt{\gamma}}{\sqrt{d} 2(1+\frac{\kappa}{2\tau})} g\|_2^2 + \zeta \frac{\|\Piper \theta\|_1}{\sqrt{d}}  \\
 \overset{\tilde{\theta}_{\perp} = \frac{\Piper \theta}{\sqrt{d}}}= \min_{\tilde{\theta}_{\perp}} &\frac{1}{d} (1+\frac{\kappa}{2\tau})\| \tilde{\theta}_{\perp} +\frac{r \sqrt{\gamma}}{2(1+\frac{\kappa}{2\tau})} g \|_2^2 + \zeta \frac{\|\Piper \theta\|_1}{d}
 \\ = &\frac{1}{d} \sum_{i=2}^d  \min_{(\tilde{\theta}_{\perp})_i}~ (1+\frac{\kappa}{2\tau})((\tilde{\theta}_{\perp})_i +\frac{r \sqrt{\gamma}}{ 2(1+\frac{\kappa}{2\tau})} g_i)^2 +  \zeta \abs{(\tilde{\theta}_{\perp})_i}
   \\ = &\frac{1}{d} 2(1+\frac{\kappa}{2\tau}) \sum_{i=2}^d  \min_{(\tilde{\theta}_{\perp})_i}~ \frac{1}{2}((\tilde{\theta}_{\perp})_i +\frac{r \sqrt{\gamma}}{ 2(1+\frac{\kappa}{2\tau})} g_i)^2 + \frac{\zeta}{2(1+\frac{\kappa}{2\tau})} \abs{(\tilde{\theta}_{\perp})_i}
    \\ = &\frac{1}{d} 2(1+\frac{\kappa}{2\tau}) \sum_{i=2}^d  \huberloss(-\frac{r \sqrt{\gamma}}{ 2(1+\frac{\kappa}{2\tau})} g_i, \frac{\zeta}{2(1+\frac{\kappa}{2\tau})})
\\ \to  & 2(1+\frac{\kappa}{2\tau}) \EE_Z ~\huberloss \left(\frac{r \sqrt{\gamma}}{ 2(1+\frac{\kappa}{2\tau})} Z, \frac{\zeta}{2(1+\frac{\kappa}{2\tau})}\right)
 \end{split}
\end{equation*}
where we solve the optimization in the fourth line with  $\huberloss$ being the Huber loss, given by ${\huberloss(x,y) = \left\{ \begin{matrix} 0.5 x^2 ~~\abs{x} \leq y \\ y (\abs{x} - 0.5 y)                                                                                       \end{matrix} \right.}$. Finally, we can conclude the proof from
\begin{equation*}
 \EE_Z ~\huberloss \left(a Z, b\right) =  \frac{a^2+b^2}{2}\erf\left(\frac{b}{\sqrt{2} a}\right) - \frac{b^2}{2} + \frac{a b}{\sqrt{2 \pi}} \exp\left(- \frac{b^2}{2 a^2} \right) .
\end{equation*}

\paragraph{Convergence}
One can check that  the optimization problems defined in Equations~\eqref{eq:thmsepeq},\eqref{eq:primalsep} are convex in the variables which we minimize over and concave in the variables which we maximize over. 
Indeed, Equation~\eqref{eq:primalsep} is immediate and 
 Equation~\eqref{eq:thmsepeq} follows straightforwardly, like before, from the
 fact that the desired convexity and concavity are satisfied for the problem defined in Equation~\eqref{eq:nuperpar2}. 
Thus, both problems defined in Equations~\eqref{eq:thmsepeq},\eqref{eq:primalsep} have unique solutions.
We note again that the optimum $\delta^\star$ in Equation \eqref{eq:delta2}
satisfies $\delta^\star = \frac{1}{\sqrt{d}} \| \Piper \theta\|$, and similarly
the optima $\nuper^\star$ and $\nupar^\star$ in Equation \eqref{eq:nuperpar2}
satisfy $\nuper^\star = \| \Piper \theta\|_2$ and $\nupar^\star = \langle
\theta,  \thetanull \rangle$. Hence we can conclude the proof as the solutions
of problems~\eqref{eq:thmsepeq}, \eqref{eq:primalsep}
concentrate asymptotically as $d,n \to \infty$ around the same optima.